%% file: arxiv.tex
\pdfoutput=1
\documentclass{article} 
\usepackage{times}
\usepackage{geometry}

\usepackage{mkolar_definitions}
\usepackage{url}
\usepackage{authblk}

\usepackage{graphicx}
\usepackage{subfigure}
\usepackage{multirow}
\usepackage{floatrow}
\newfloatcommand{capbtabbox}{table}[][\FBwidth]

\usepackage{algorithm}
\usepackage{algorithmic}

\usepackage{enumerate}

\usepackage{amsthm,amsmath,amssymb}

\def\tU {\tilde{U}}
\newenvironment{packed_enum}{
\begin{enumerate}
\setlength{\itemsep}{1pt}
\setlength{\parskip}{0pt}
\setlength{\parsep}{0pt}
}{\end{enumerate}}

\newcommand{\figsize}{0.3}
\newcommand{\figsizeb}{0.45}

\begin{document} 

\title{Low-Rank Matrix and Tensor Completion via Adaptive Sampling}

\author[1]{
Akshay Krishnamurthy
\thanks{akshaykr@cs.cmu.edu}}
\author[2]{
Aarti Singh
\thanks{aarti@cs.cmu.edu}}

\affil[1]{Computer Science Department\\
Carnegie Mellon University}
\affil[2]{Machine Learning Department\\
Carnegie Mellon University}

\maketitle

\input{abstract}
\input{intro_2}
\input{related}
\input{exact}

\input{noisy_mc}

\input{experiments}
\input{discussion}
\bibliography{bibliography}
\bibliographystyle{plain}

\appendix
\input{appendix}

\input{app_css}
\input{app_lower}

\input{app_inc}
\input{app_noise}

\end{document}

%% file: abstract.tex
\begin{abstract} 
We study low rank matrix and tensor completion and propose novel algorithms that employ adaptive sampling schemes to obtain strong performance guarantees. 
Our algorithms exploit adaptivity to identify entries that are highly informative for learning the column space of the matrix (tensor) and consequently, our results hold even when the row space is highly coherent, in contrast with previous analyses.
In the absence of noise, we show that one can exactly recover a $n \times n$ matrix of rank $r$ from merely $\Omega(n r^{3/2}\log(r))$ matrix entries.
We also show that one can recover an order $T$ tensor using $\Omega(n r^{T-1/2}T^2 \log(r))$ entries.
For noisy recovery, our algorithm consistently estimates a low rank matrix corrupted with noise using $\Omega(n r^{3/2} \textrm{polylog}(n))$ entries.
We complement our study with simulations that verify our theory and demonstrate the scalability of our algorithms.
\end{abstract} 

%% file: intro_2.tex
\section{Introduction}

Recently, the machine learning and signal processing communities have focused considerable attention toward understanding the benefits of adaptive sensing.
This theme is particularly relevant to modern data analysis, where adaptive sensing has emerged as an efficient alternative to obtaining and processing the large data sets associated with scientific investigation.
These empirical observations have lead to a number of theoretical studies characterizing the performance gains offered by adaptive sensing over conventional, passive approaches. 
In this work, we continue in that direction and study the role of adaptive data acquisition in low rank matrix and tensor completion problems. 

Our study is motivated not only by prior theoretical results in favor of adaptive sensing but also by several applications where adaptive sensing is feasible.
In recommender systems, obtaining a measurement amounts to asking a user about an item, an interaction that has been deployed in production systems.
Another application pertains to network tomography, where a network operator is interested in inferring latencies between hosts in a communication network while injecting few packets into the network. 
The operator, being in control of the network, can adaptively sample the matrix of pair-wise latencies, potentially reducing the total number of measurements.
In particular, the operator can obtain full columns of the matrix by measuring from one host to all others, a sampling strategy we will exploit in this paper.

Yet another example centers around gene expression analysis, where the object of interest is a matrix of expression levels for various genes across a number of conditions. 
There are typically two types of measurements: low-throughput assays provide highly reliable measurements of single entries in this matrix while high-throughput microarrays provide expression levels of all genes of interest across operating conditions, thus revealing entire columns. 
The completion problem can be seen as a strategy for learning the expression matrix from both low- and high-throughput data while minimizing the total measurement cost.

\subsection{Contributions}
We develop algorithms with theoretical guarantees for three low-rank completion problems. 
The algorithms find a small subset of columns of the matrix (tensor) that can be used to reconstruct or approximate the matrix (tensor). 
We exploit adaptivity to focus on highly informative columns, and this enables us to do away with the usual incoherence assumptions on the row-space while achieving competitive (or in some cases better) sample complexity bounds.
Specifically our results are:
\begin{packed_enum}
\item{} In the absence of noise, we develop a streaming algorithm that enjoys both low sample requirements and computational overhead.
In the matrix case, we show that $\Omega(nr^{3/2} \log r)$ adaptively chosen samples are sufficient for exact recovery, improving on the best known bound of $\Omega(nr^2 \log^2 n)$ in the passive setting~\cite{recht2011simpler}.
This also gives the first guarantee for matrix completion with coherent row space. 
\item{} In the tensor case, we establish that $\Omega(n r^{T-1/2}T^2 \log r)$ adaptively chosen samples are sufficient for recovering a $n \times \ldots \times n$ order $T$ tensor of rank $r$.
We complement this with a necessary condition for tensor completion under random sampling, showing that our adaptive strategy is competitive with \emph{any} passive algorithm. 
These are the first sample complexity upper and lower bounds for exact tensor completion.
\item{} 
In the noisy matrix completion setting, we modify the adaptive column subset selection algorithm of Deshpande \emph{et al.}~\cite{deshpande2006matrix} to give an algorithm that finds a rank-$r$ approximation to a matrix using $\Omega(nr^{3/2} \textrm{polylog}(n))$ samples. 
As before, the algorithm does not require an incoherent row space but we are no longer able to process the matrix sequentially.
\item{} Along the way, we improve on existing results for subspace detection from missing data, the problem of testing if a partially observed vector lies in a known subspace.
\end{packed_enum}

%% file: related.tex
\section{Related Work}
The matrix completion problem has received considerable attention in recent years. 
A series of papers \cite{candes2009exact,candes2010power,gross2011recovering,recht2011simpler}, culminating in Recht's elegent analysis of the nuclear norm minimization program, address the exact matrix completion problem through the framework of convex optimization, establishing that $\Omega((n_1+n_2)r \max\{\mu_0, \mu_1^2\}\log^2(n_2))$ randomly drawn samples are sufficient to exactly identify an $n_1\times n_2$ matrix with rank $r$. 
Here $\mu_0$ and $\mu_1$ are parameters characterizing the \emph{incoherence} of the row and column spaces of the matrix, which we will define shortly. 
Candes and Tao \cite{candes2010power} proved that under random sampling $\Omega(n_1 r \mu_0\log(n_2))$ samples are necessary, showing that nuclear norm minimization is near-optimal.

The noisy matrix completion problem has also received considerable attention~\cite{candes2010matrix,keshavan2010matrix,negahban2012restricted}. 
The majority of these results also involve some parameter that quantifies how much information a single observation reveals, in the same vein as incoherence.

Tensor completion, a natural generalization of matrix completion, is less studied. 
One challenge stems from the NP-hardness of computing most tensor decompositions, pushing researchers to study alternative structure-inducing norms in lieu of the nuclear norm~\cite{gandy2011tensor,tomioka2010estimation}.
Both papers derive algorithms for tensor completion, but neither provide sample complexity bounds for the noiseless case.

Our approach involves adaptive data acquisition, and consequently our work is closely related to a number of papers focusing on using adaptive measurements to estimate a sparse vector~\cite{davenport2012compressive,haupt2009compressive}.
In these problems, specifically, problems where the sparsity basis is known a priori, we have a reasonable understanding of how adaptive sampling can lead to performance improvements. 
As a low rank matrix is sparse in its unknown eigenbasis, the completion problem is coupled with learning this basis, which poses a new challenge for adaptive sampling procedures.

Another relevant line of work stems from the \emph{matrix approximations} literature. 
Broadly speaking, this research is concerned with efficiently computing a structured matrix, i.e. sparse or low rank, that serves as a good approximation to a fully observed input matrix.
Two methods that apply to the missing data setting are the Nystrom method~\cite{gittens2011spectral,kumar2012sampling} and entrywise subsampling~\cite{achlioptas2007fast}.
While the sample complexity bounds match ours, the analysis for the Nystrom method has focused on positive-semidefinite kernel matrices and requires incoherence of both the row and column spaces.
On the other hand, entrywise subsampling is applicable, but the guarantees are weaker than ours.

It is also worth briefly mentioning the vast body of literature on column subset selection, the task of approximating a matrix by projecting it onto a few of its columns. 
While the best algorithms, namely volume sampling~\cite{guruswami2012optimal} and sampling according to statistical leverages~\cite{boutsidis2009improved}, do not seem to be readily applicable to the missing data setting, some algorithms are. 
Indeed our procedure for noisy matrix completion is an adaptation of an existing column subset selection procedure~\cite{deshpande2006matrix}. 

Our techniques 
are also closely related to ideas employed for subspace detection -- testing whether a vector lies in a known subspace -- and subspace tracking -- learning a time-evolving low-dimensional subspace from vectors lying close to that subspace. 
Balzano \emph{et al.} \cite{balzano2010high} prove guarantees for subspace detection with known subspace and a partially observed vector, and we will improve on their result en route to establishing our guarantees. 
Subspace tracking from partial information has also been studied~\cite{he2011online}, but little is known theoretically about this problem.

\section{Definitions and Preliminaries}
Before presenting our algorithms, we clarify some notation and definitions. 
Let $M \in \mathbb{R}^{n_1 \times n_2}$ be a rank $r$ matrix with singular value decomposition $U\Sigma V^T$. 
Let $c_1, \ldots c_{n_2}$ denote the columns of $M$.

Let $\mathbb{M} \in \mathbb{R}^{n_1 \times \ldots \times n_T}$ denote an order $T$ tensor with canonical decomposition: 
\begin{eqnarray}
\mathbb{M} = \sum_{k=1}^ra_k^{(1)} \otimes a_k^{(2)} \otimes \ldots \otimes a_k^{(T)}
\end{eqnarray}
where $\otimes$ is the outer product. 
Define $\textrm{rank}(\mathbb{M})$ to be the smallest value of $r$ that establishes this equality. 
Note that the vectors $\{a_k^{(t)}\}_{k=1}^r$ need not be orthogonal, nor even linearly independent.

The mode-$t$ subtensors of $\mathbb{M}$, denoted $\mathbb{M}_i^{(t)}$, are order $T-1$ tensors obtained by fixing the $i$th coordinate of the $t$-th mode. 
For example, if $\mathbb{M}$ is an order $3$ tensor, then $\mathbb{M}_{i}^{(3)}$ are the frontal slices.

We represent a $d$-dimensional subspace $U \subset \mathbb{R}^n$ as a set of orthonormal basis vectors $U = \{u_i\}_{i=1}^d$ and in some cases as $n \times d$ matrix whose columns are the basis vectors. 
The interpretation will be clear from context.
Define the \textbf{orthogonal projection} onto $U$ as $\mathcal{P}_Uv = U(U^TU)^{-1}U^Tv$.

For a set $\Omega \subset [n]$\footnote{We write $[n]$ for $\{1, \ldots, n\}$}, $c_{\Omega} \in \mathbb{R}^{|\Omega|}$ is the vector whose elements are $c_i, i \in \Omega$ indexed lexicographically. 
Similarly the matrix $U_{\Omega} \in \mathbb{R}^{|\Omega| \times d}$ has rows indexed by $\Omega$ lexicographically. 
Note that if $U$ is a orthobasis for a subspace, $U_{\Omega}$ is a $|\Omega| \times d$ matrix with columns $u_{i \Omega}$ where $u_i \in U$, rather than a set of orthonormal basis vectors.
In particular, the matrix $U_{\Omega}$ need not have orthonormal columns. 

These definitions extend to the tensor setting with slight modifications. 
We use the $\texttt{vec}$ operation to unfold a tensor into a single vector and define the inner product $\langle x, y \rangle = \texttt{vec}(x)^T \texttt{vec}(y)$. 
For a subspace $U \subset \mathbb{R}^{\otimes n_i}$, we write it as a $\left(\prod n_i\right) \times d$ matrix whose columns are $\texttt{vec}(u_i)$, $u_i \in U$. 
We can then define projections and subsampling as we did in the vector case.

As in recent work on matrix completion \cite{candes2010power,recht2011simpler}, we will require a certain amount of incoherence between the column space associated with $M$ ($\mathbb{M}$) and the standard basis. 
\begin{definition}
The \textbf{coherence} of an $r$-dimensional subspace $U \subset \mathbb{R}^n$
is:
\begin{eqnarray}
\mu(U) \triangleq \frac{n}{r} \max_{1 \le j \le n} ||\mathcal{P}_U e_j||^2
\end{eqnarray}
where $e_j$ denotes the $j$th standard basis element.
\end{definition}
In previous analyses of matrix completion, the \emph{incoherence assumption} is that both the row and column spaces of the matrix have coherences upper bounded by $\mu_0$. 
When both spaces are incoherent, each entry of the matrix reveals roughly the same amount of information, so there is little to be gained from adaptive sampling, which typically involves looking for highly informative measurements.
Thus the power of adaptivity for these problems should center around relaxing the incoherence assumption, which is the direction we take in this paper.
Unfortunately, even under adaptive sampling, it is impossible to identify a rank one matrix that is zero in all but one entry without observing the entire matrix, implying that we cannot completely eliminate the assumption.
Instead, we will retain incoherence on the column space, but remove the restrictions on the row space.


%% file: exact.tex
\section{Exact Completion Problems}
In the matrix case, our sequential algorithm builds up the column space of the matrix by selecting a few columns to observe in their entirety.
In particular, we maintain a candidate column space $\tU$ and test whether a column $c_i$ lives in $\tU$ or not, choosing to completely observe $c_i$ and add it to $\tU$ if it does not. 
Balzano \emph{et al.}~\cite{balzano2010high} observed that we can perform this test with a subsampled version of $c_i$, meaning that we can recover the column space using few samples.
Once we know the column space, recovering the matrix, even from few observations, amounts to solving determined linear systems. 


For tensors, the algorithm becomes recursive in nature.
At the outer level of the recursion, the algorithm maintains a candidate subspace $\mathcal{U}$ for the mode $T$ subtensors $\mathbb{M}_{i}^{(T)}$.
For each of these subtensors, we test whether $\mathbb{M}_{i}^{(T)}$ lives in $\mathcal{U}$ and recursively complete that subtensor if it does not.
Once we complete the subtensor, we add it to $\mathcal{U}$ and proceed at the outer level.
When the subtensor itself is just a column; we observe the columns in its entirety.


\begin{algorithm}[t]
\caption{Sequential Tensor Completion $(\mathbb{M}, \{m_t\}_{t=1}^T)$}
\begin{packed_enum}
\item{} Let $\mathcal{U} = \emptyset$.
\item{} Randomly draw entries $\Omega \subset \prod_{t=1}^{T-1} [n_t]$ uniformly with replacement w. p. $m_T/\prod_{t=1}^{T-1}n_t$. 
\item{} For each mode-$T$ subtensor $\mathbb{M}_i^{(T)}$ of $\mathbb{M}$, $i \in
  [n_T]$:
\begin{packed_enum}
\item{} If $||\mathbb{M}_{i\Omega}^{(T)} - \mathcal{P}_{\mathcal{U}_{\Omega}}\mathbb{M}_{i\Omega}^{(t)}||_2^2 > 0$:
\begin{packed_enum}
\item{} $\hat{\mathbb{M}}_i^{(T)} \gets $ recurse on ($\mathbb{M}_i^{(T)}$, $\{m_t\}_{t=1}^{T-1}$)
\item{} $\mathbb{U}_i \gets \frac{\mathcal{P}_{\mathcal{U}^\perp} \hat{\mathbb{M}}_i^{(T)}}{||\mathcal{P}_{\mathcal{U}^\perp} \hat{\mathbb{M}}_i^{(T)}||}$. $\mathcal{U} \gets \mathcal{U} \cup \mathbb{U}_i$.
\end{packed_enum}
\item{} Otherwise $\hat{\mathbb{M}}_i^{(T)} \gets \mathcal{U}(\mathcal{U}_{\Omega}^T
  \mathcal{U}_{\Omega})^{-1} \mathcal{U}_{\Omega} \mathbb{M}_{i\Omega}^{(T)}$
\end{packed_enum}
\item{} Return $\hat{\mathbb{M}}$ with mode-$T$ subtensors $\hat{\mathbb{M}_i}^{(T)}$.
\end{packed_enum}
\label{alg:tc}
\end{algorithm}

The pseudocode of the algorithm is given in Algorithm~\ref{alg:tc}.
Our first main result characterizes the performance of the tensor completion algorithm.
We defer the proof to the appendix.
\begin{theorem}
Let $\mathbb{M} = \sum_{i=1}^r \otimes_{t=1}^Ta_j^{(t)}$ be a rank $r$ order-$T$ tensor with subspaces $A^{(t)} = \textrm{span}(\{a_j^{(t)}\}_{j=1}^r)$.
Suppose that all of $A^{(1)}, \ldots  A^{(T-1)}$ have coherence bounded above by $\mu_0$.
Set $m_t = 36 r^{t-1/2}\mu_0^{t-1} \log(2r/\delta)$ for each $t$.
Then with probability $\ge 1-5\delta T r^T$, Algorithm~\ref{alg:tc} exactly recovers $\mathbb{M}$ and has expected sample complexity
\begin{eqnarray}
36(\sum_{t=1}^T n_t)r^{T-1/2}\mu_0^{T-1}\log(2r/\delta)
\end{eqnarray}
\label{thm:tc}
\end{theorem}

In the special case of a $n\times \ldots \times n$ tensor of order $T$, the algorithm succeeds with high probability using $\Omega(nr^{T-1/2}\mu_0^{T-1}T^2 \log(Tr/\delta))$ samples, exhibiting a linear dependence on the tensor dimensions. 
In comparison, the only guarantee we are aware of shows that $\Omega\left(\left(\prod_{t=2}^{T_1}n_t\right) r\right)$ samples are sufficient for consistent estimation of a noisy tensor, exhibiting a much worse dependence on tensor dimension~\cite{tomioka2011statistical}.
In the noiseless scenario, one can unfold the tensor into a $n_1 \times \prod_{t=2}^{T} n_t$ matrix and apply any matrix completion algorithm.
Unfortunately, without exploiting the additional tensor structure, this approach will scale with $\prod_{t=2}^Tn_t$, which is similarly much worse than our guarantee.
Note that the na\"{i}ve procedure that does not perform the recursive step has sample complexity scaling with the product of the dimensions and is therefore much worse than the our algorithm.


The most obvious specialization of Theorem~\ref{thm:tc} is to the matrix completion problem:
\begin{corollary}
Let $M := U\Sigma V^T \in \mathbb{R}^{n_1\times n_2}$ have rank $r$, and fix $\delta > 0$.
Assume $\mu(U) \le \mu_0$.
Setting $m \triangleq m_2 \ge 36 r^{3/2} \mu_0 \log(\frac{2r}{\delta})$, the sequential algorithm exactly recovers $M$ with probability at least $1 - 4r\delta+\delta$ while using in expectation
\begin{eqnarray}
36n_2r^{3/2} \mu_0\log(2r/\delta) + rn_1
\end{eqnarray}
observations.
The algorithm runs in $O(n_1n_2r + r^3m)$ time.
\label{cor:mc}
\end{corollary}

A few comments are in order.
Recht \cite{recht2011simpler} guaranteed exact recovery for the nuclear norm minimization procedure as long as the number of observations exceeds $32(n_1+n_2) r \max\{\mu_0, \mu_1^2\}\beta \log^2(2n_2)$ where $\beta$ controls the probability of failure and $||UV^T||_{\infty} \le \mu_1 \sqrt{r/(n_1n_2)}$ with $\mu_1$ as another coherence parameter.
Without additional assumptions, $\mu_1$ can be as large as $\mu_0\sqrt{r}$.
In this case, our bound improves on his in its the dependence on $r, \mu_0$ and logarithmic terms.

The Nystrom method can also be applied to the matrix completion problem, albeit under non-uniform sampling. 
Given a PSD matrix, one uses a randomly sampled set of columns and the corresponding rows to approximate the remaining entries.
Gittens showed that if one samples $O(r \log r)$ columns, then one can exactly reconstruct a rank $r$ matrix~\cite{gittens2011spectral}.
This result requires incoherence of both row and column spaces, so it is more restrictive than ours.
Almost all previous results for exact matrix completion require incoherence of both row and column spaces.

The one exception is a recent paper by Chen \emph{et al.} that we became aware of while preparing the final version of this work~\cite{chen2013coherent}.
They show that sampling the matrix according to statistical leverages of the rows and columns can eliminate the need for incoherence assumptions.
Specifically, when the matrix has incoherent column space, they show that by first estimating the leverages of the columns, sampling the matrix according to this distribution, and then solving the nuclear norm minimization program, one can recover the matrix with $\Omega(n r \mu_0 \log^2 n)$ samples. 
Our result improves on theirs when $r$ is small compared to $n$, specifically when $\sqrt{r}\log r \le \log^2n$, which is common.

Our algorithm is also very computationally efficient. 
Existing algorithms involve successive singular value decompositions ($O(n_1n_2r)$ per iteration), resulting in much worse running times.


The key ingredient in our proofs is a result pertaining to subspace detection, the task of testing if a subsampled vector lies in a subspace.
This result, which improves over the results of Balzano \emph{et al.}~\cite{balzano2010high}, is crucial in obtaining our sample complexity bounds, and may be of independent interest.
\begin{theorem}
Let $U$ be a $d$-dimensional subspace of $\mathbb{R}^n$ and $y = x+v$ where $x \in U$ and $v \in U^\perp$.
Fix $\delta > 0$, $m \ge \frac{8}{3} d \mu(U) \log\left(\frac{2d}{\delta}\right)$ and let $\Omega$ be an index set with entries sampled uniformly with replacement with probability $m/n$. 
Then with probability at least $1-4\delta$:
\begin{eqnarray}
\frac{m(1-\alpha) - d \mu(U) \frac{\beta}{(1-\gamma)}}{n} ||v||_2^2 \le ||y_{\Omega} - \mathcal{P}_{U_{\Omega}} y_{\Omega}||_2^2 \le (1+\alpha)\frac{m}{n}||v||_2^2
\label{eq:detection_lwr_bd_new}
\end{eqnarray}
Where $\alpha = \sqrt{2 \frac{\mu(v)}{m} \log(1/\delta)} + 2 \frac{\mu(v)}{3m}\log(1/\delta)$, $\beta= 6 \log(d/\delta) + \frac{4}{3}\frac{d\mu(v)}{m}\log^2(d/\delta)$, $\gamma = \sqrt{\frac{8d\mu(U)}{3m}\log(2d/\delta)}$ and $\mu(v) = n ||v||_{\infty}^2/||v||_2^2$.
\label{thm:laura_new}
\end{theorem}
This theorem shows that if $m = \Omega(\max\{\mu(v), d \mu(U), d \sqrt{\mu(U)\mu(v)}\}\log d)$ then the orthogonal projection from missing data is within a constant factor of the fully observed one.
In contrast, Balzano \emph{et al.}~\cite{balzano2010high} give a similar result that requires $m = \Omega(\max\{\mu(v)^2, d \mu(U), d \mu(U) \mu(v)\} \log d)$ to get a constant factor approximation. 
In the matrix case, this improved dependence on incoherence parameters brings our sample complexity down from $nr^2 \mu_0^2\log r$ to $nr^{3/2} \mu_0 \log r$. 
We conjecture that this theorem can be further improved to eliminate another $\sqrt{r}$ factor from our final bound.

\subsection{Lower Bounds for Uniform Sampling}

We adapt the proof strategy of Candes and Tao~\cite{candes2010power} to the tensor completion problem and establish the following lower bound for uniform sampling:

\begin{theorem}[Passive Lower Bound]
Fix $1 \le m, r \le \min_{t}n_t$ and $\mu_0 > 1$. Fix $0 < \delta < 1/2$ and suppose that
we do not have the condition:
\begin{eqnarray}
-\log\left(1 - \frac{m}{\prod_{i=1}^Tn_i}\right) \ge
\frac{\mu_0^{T-1}r^{T-1}}{\prod_{i=2}^Tn_i}\log\left(\frac{n_1}{2\delta}\right)
\label{eq:lower_bound}
\end{eqnarray}
Then there exist infinitely many pairs of distinct $n_1 \times \ldots \times n_T$ order-$T$ tensors $\mathbb{M} \ne \mathbb{M}'$ of rank $r$ with coherence parameter $ \le \mu_0$ such that $\mathcal{P}_{\Omega}(\mathbb{M}) = \mathcal{P}_{\Omega}(\mathbb{M}')$ with probability at least $\delta$.
Each entry is observed independently with probability $T = \frac{m}{\prod_{i=1}^T n_i}$.
\label{thm:lower_bound}
\end{theorem}

Theorem~\ref{thm:lower_bound} implies that as long as the right hand side of Equation~\ref{eq:lower_bound} is at most $\epsilon < 1$, and:
\begin{eqnarray}
m \le n_1 r^{T-1}\mu_0^{T-1}\log\left(\frac{n_1}{2\delta}\right) (1-\epsilon/2)
\label{eq:lower_bound_translated}
\end{eqnarray}
then with probability at least $\delta$ there are infinitely many matrices that agree on the observed entries.
This gives a necessary condition on the number of samples required for tensor completion.
Note that when $T=2$ we recover the known lower bound for matrix completion.


Theorem~\ref{thm:lower_bound} gives a necessary condition under uniform sampling.
Comparing with Theorem~\ref{thm:tc} shows that our procedure outperforms any passive procedure in its dependence on the tensor dimensions.
However, our guarantee is suboptimal in its dependence on $r$.
The extra factor of $\sqrt{r}$ would be eliminated by a further improvement to Theorem~\ref{thm:lower_bound}, which we conjecture is indeed possible.

For adaptive sampling, one can obtain a lower bound via a parameter counting argument.
Observing the $(i_1, \ldots, i_T)$th entry leads to a polynomial equation of the form $\sum_{k} \prod_{t} a_k^{(t)}(i_t) = M_{i_1,\ldots,i_T}$.
If $m < r(\sum_t n_t)$, this system is underdetermined showing that $\Omega((\sum_t n_t)r)$ observations are necessary for exact recovery, even under adaptive sampling. 
Thus, our algorithm enjoys sample complexity with optimal dependence on matrix dimensions.

%% file: noisy_mc.tex
\section{Noisy Matrix Completion}
Our algorithm for noisy matrix completion is an adaptation of the column subset selection (CSS) algorithm analyzed by Deshpande \emph{et al.} \cite{deshpande2006matrix}. 
The algorithm builds a candidate column space in rounds; at each round it samples additional columns with probability proportional to their projection on the orthogonal complement of the candidate column space.

To concretely describe the algorithm, suppose that at the beginning of the $l$th round we have a candidate subspace $U_l$. 
Then in the $l$th round, we draw $s$ additional columns according to the distribution where the probability of drawing the $i$th column is proportional to $||\Pcal_{U_l^\perp}c_i||_2^2$.
Observing these $s$ columns in full and then adding them to the subspace $U_l$ gives the candidate subspace $U_{l+1}$ for the next round.
We initialize the algorithm with $U_1 = \emptyset$.
After $L$ rounds, we approximate each column $c$ with $\hat{c} = U_L(U_{L\Omega}^TU_{L\Omega})^{-1}U^T_{L\Omega} c_\Omega$ and concatenate these estimates to form $\hat{M}$.

The challenge is that the algorithm cannot compute the sampling probabilities without observing entries of the matrix. 
However, our results show that with reliable estimates, which can be computed from few observations, the algorithm still performs well.

We assume that the matrix $M \in \mathbb{R}^{n_1 \times n_2}$ can be decomposed as a rank $r$ matrix $A$ and a random gaussian matrix $R$ whose entries are independently drawn from $\mathcal{N}(0, \sigma^2)$. 
We write $A = U \Sigma V^T$ and assume that $\mu(U) \le \mu_0$.
As before, the incoherence assumption is crucial in guaranteeing that one can estimate the column norms, and consequently sampling probabilities, from missing data.

\begin{theorem}
Let $\Omega$ be the set of all observations over the course of the algorithm, let $U_L$ be the subspace obtained after $L = log(n_1n_2)$ rounds and $\hat{M}$ be the matrix whose columns ${\hat{c}_i = U_L(U_{L\Omega}^TU_{L\Omega})^{-1}U_{L\Omega}^Tc_{\Omega i}}$.
Then there are constants $c_1, c_2$ such that:
\[
||A - \hat{M}||_F^2 \le \frac{c_1}{(n_1n_2)}||A||_F^2 + c_2 ||R_{\Omega}||_F^2
\]
$\hat{M}$ can be computed from $\Omega((n_1+n_2)r^{3/2}\mu(U) \textrm{polylog}(n_1n_2))$ observations.
In particular, if $||A||_F^2 = 1$ and $R_{ij} \sim \mathcal{N}(0, \sigma^2/(n_1n_2))$, then there is a constant $c_\star$ for which:
\[
||A - \hat{A}||_F^2 \le \frac{c_\star}{n_1n_2}\left(1 + \sigma^2 \left((n_1+n_2)r^{3/2} \mu(U) \textrm{polylog}(n_1n_2)\right)\right)
\]
\label{cor:noisy_mc}
\vspace{-0.25cm}
\end{theorem}

The main improvement in the result is in relaxing the assumptions on the underlying matrix $A$.
Existing results for noisy matrix completion require that the energy of the matrix is well spread out across both the rows and the columns (i.e. incoherence), and the sample complexity guarantees deteriorate significantly without such an assumption~\cite{candes2010matrix,keshavan2010matrix}. 
As a concrete example, Negahban and Wainwright~\cite{negahban2012restricted} use a notion of spikiness, measured as $\sqrt{n_1n_2}\frac{||A||_{\infty}}{||A||_F}$ which can be as large as $\sqrt{n_2}$ in our setup, e.g. when the matrix is zero except for on one column and constant across that column.

The choices of $||A||_F^2 = 1$ and noise variance rescaled by $\frac{1}{n_1n_2}$ enable us to compare our results with related work~\cite{negahban2012restricted}. 
Thinking of $n_1 = n_2 = n$ and the incoherence parameter as a constant, our results imply consistent estimation as long as $\sigma^2 = \omega\left(\frac{n}{r^2 \textrm{polylog}(n)}\right)$. 
On the other hand, thinking of the spikiness parameter as a constant,~\cite{negahban2012restricted} show that the error is bounded by $\frac{\sigma^2 nr \log n}{m}$ where $m$ is the total number of observations. 
Using the same number of samples as our procedure, their results implies consistency as long as $\sigma^2 = \omega(r \textrm{polylog}(n))$. 
For small $r$ (i.e. $r = O(1)$), our noise tolerance is much better, but their results apply even with fewer observations, while ours do not.

%% file: experiments.tex
\vspace{-0.25cm}
\section{Simulations}
\label{sec:simulations}
\begin{figure*}
\centering{
\subfigure[]{
\includegraphics[scale=\figsize]{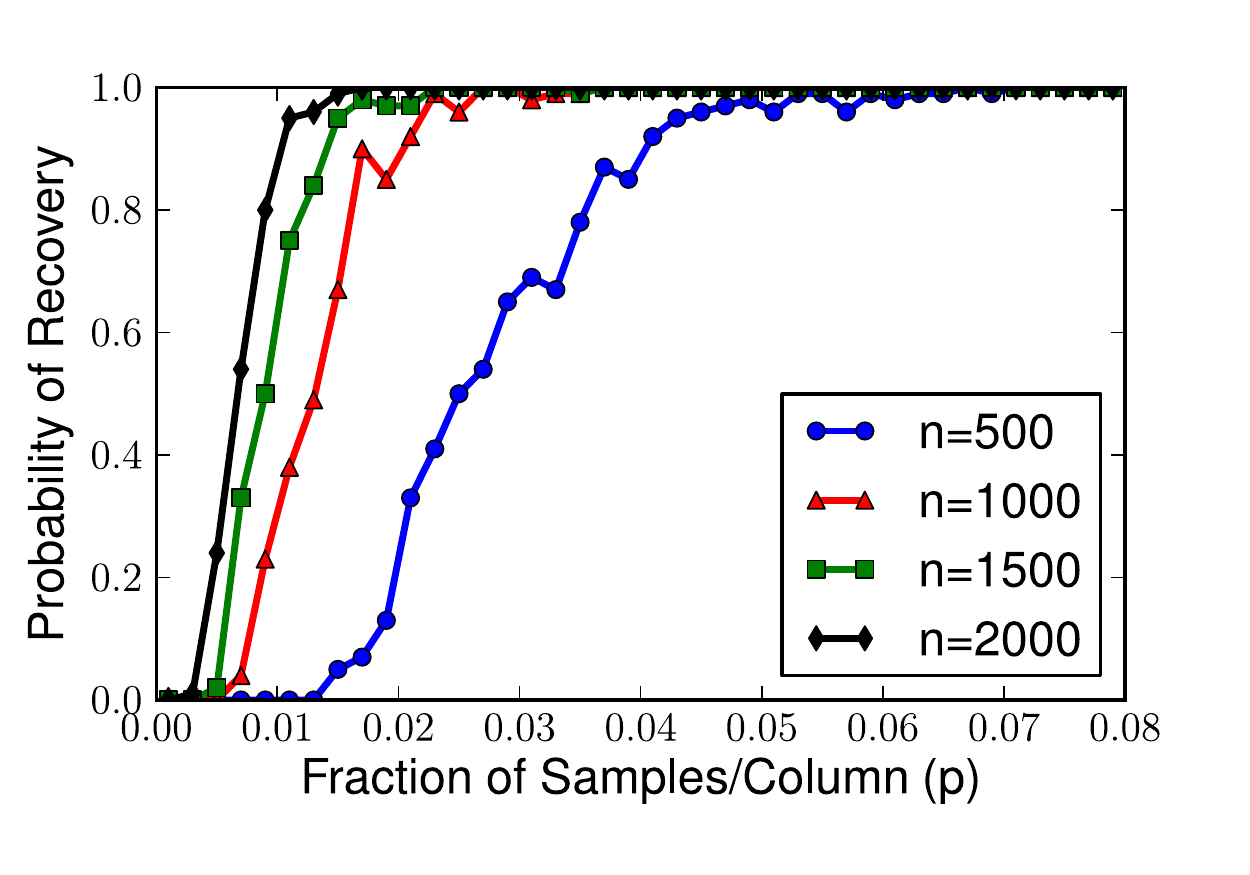}
\label{fig:unscaled_mc}
}
\subfigure[]{
\includegraphics[scale=\figsize]{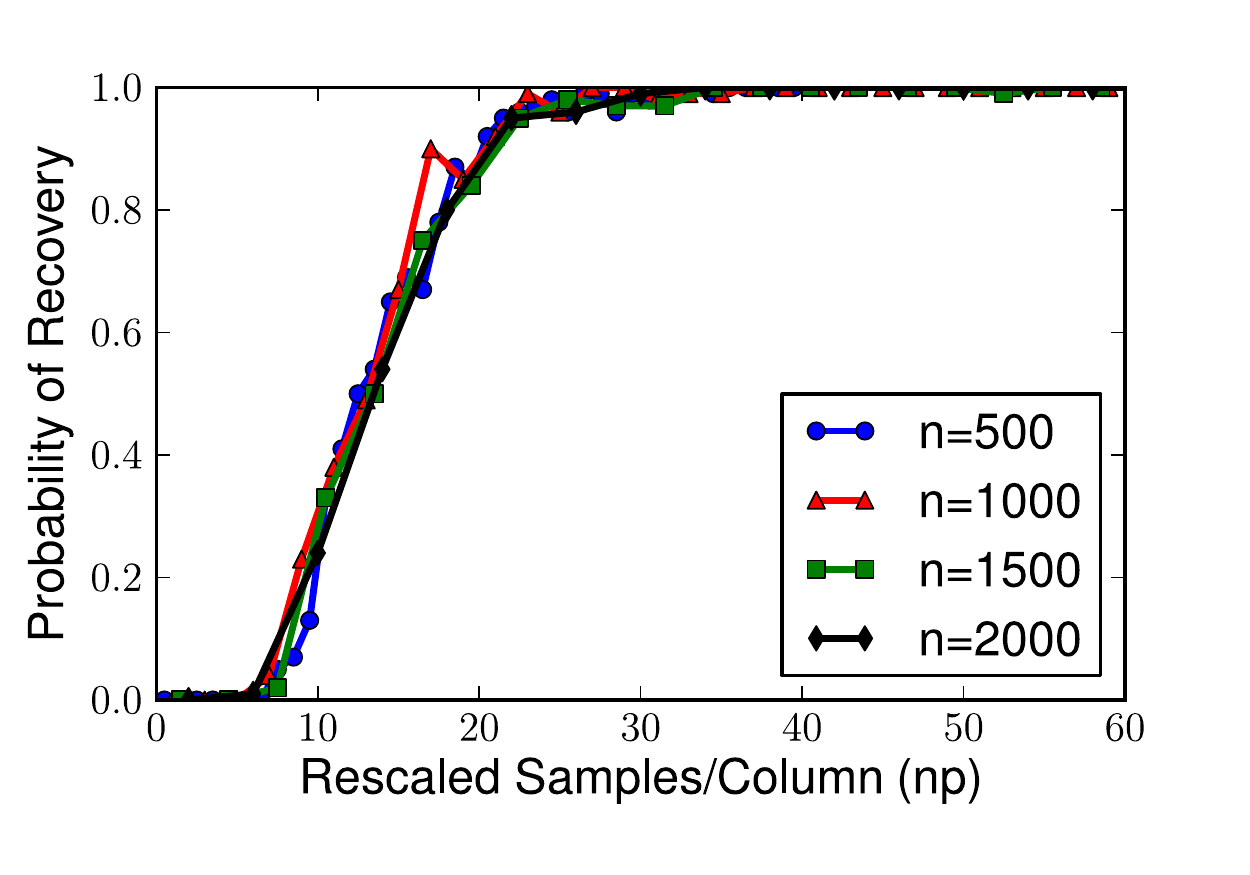}
\label{fig:linear_mc}
}
\subfigure[]{
\includegraphics[scale=\figsize]{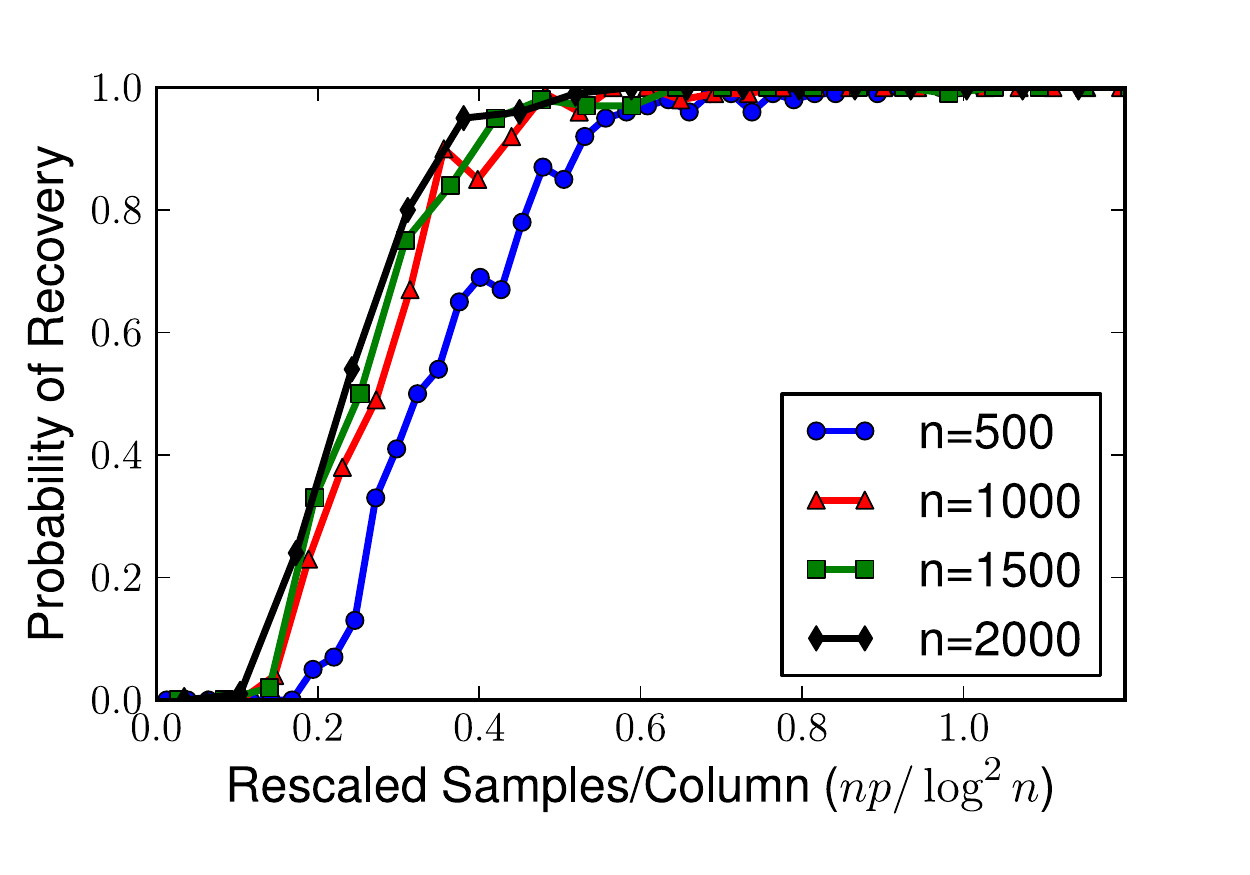}
\label{fig:logarithmic_mc}
}\\
\vspace{-0.4cm}
\subfigure[]{
\includegraphics[scale=\figsize]{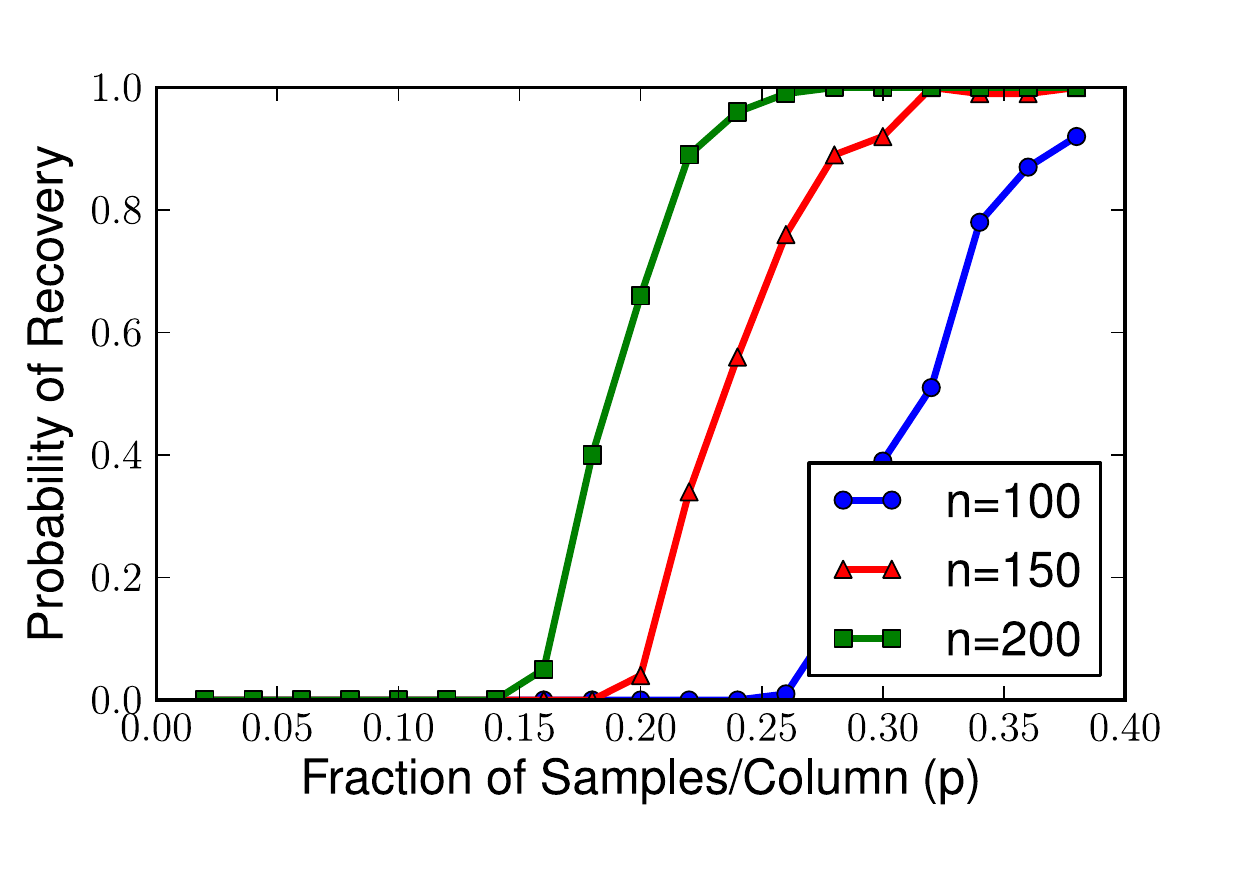}
\label{fig:unscaled_svt}
}
\subfigure[]{
\includegraphics[scale=\figsize]{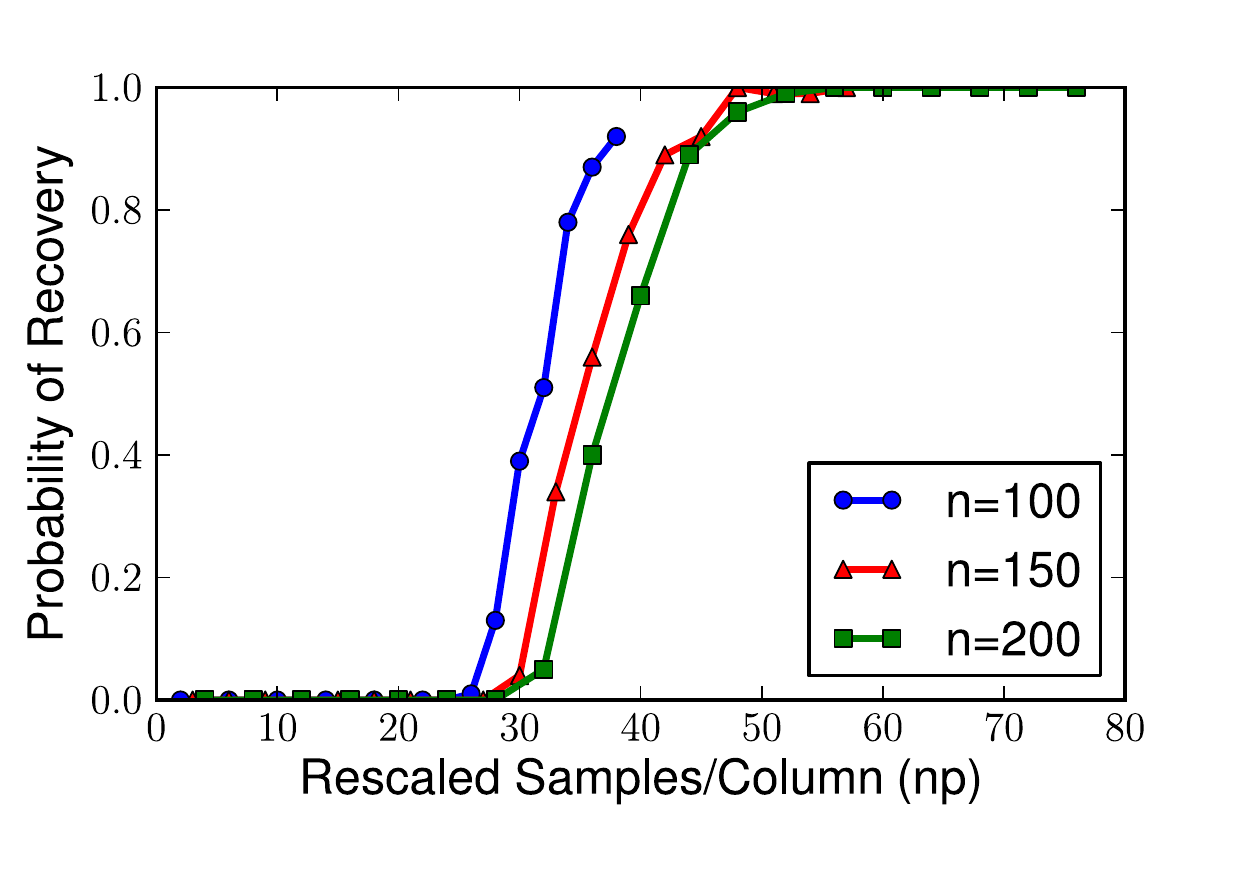}
\label{fig:linear_svt}
}
\subfigure[]{
\includegraphics[scale=\figsize]{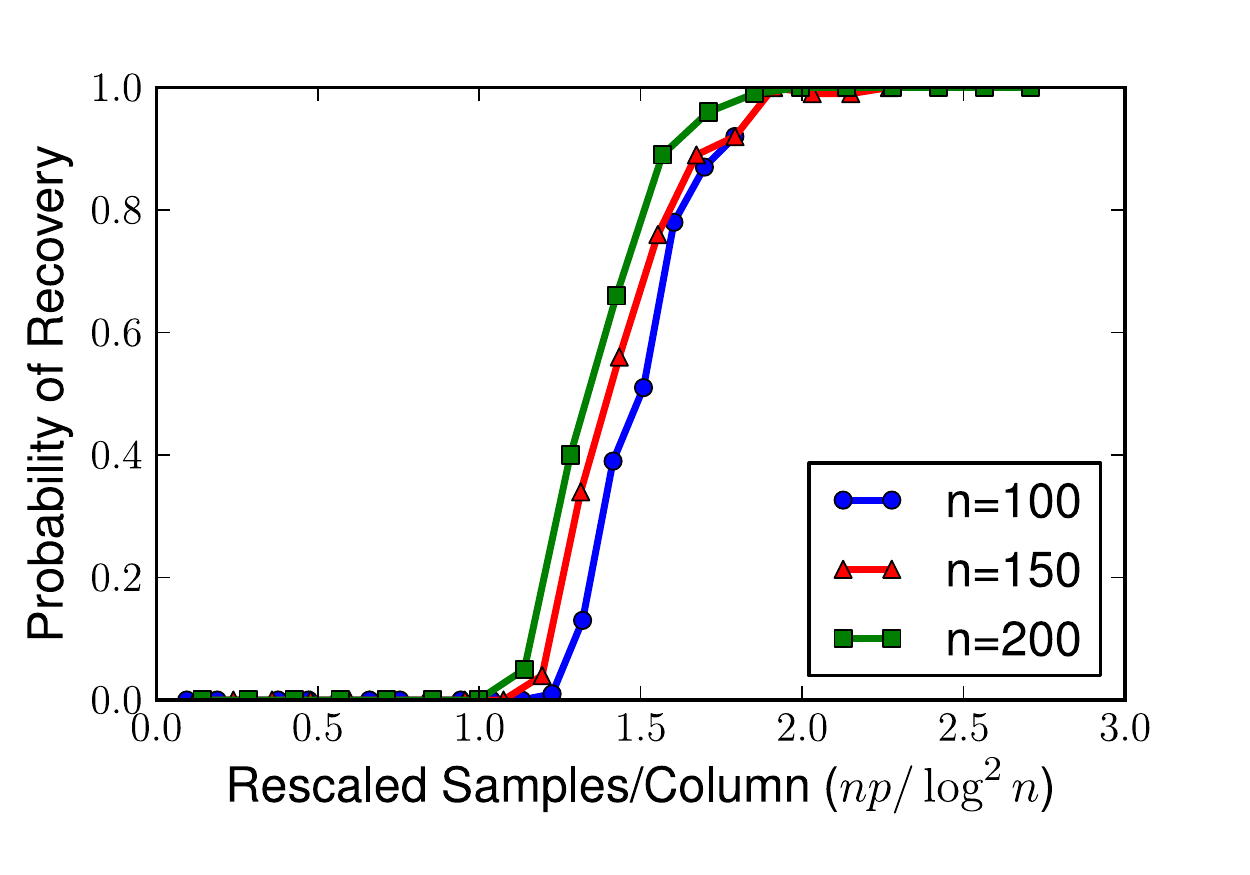}
\label{fig:logarithmic_svt}
}\\
\vspace{-0.4cm}
\subfigure[]{
\includegraphics[scale=\figsize]{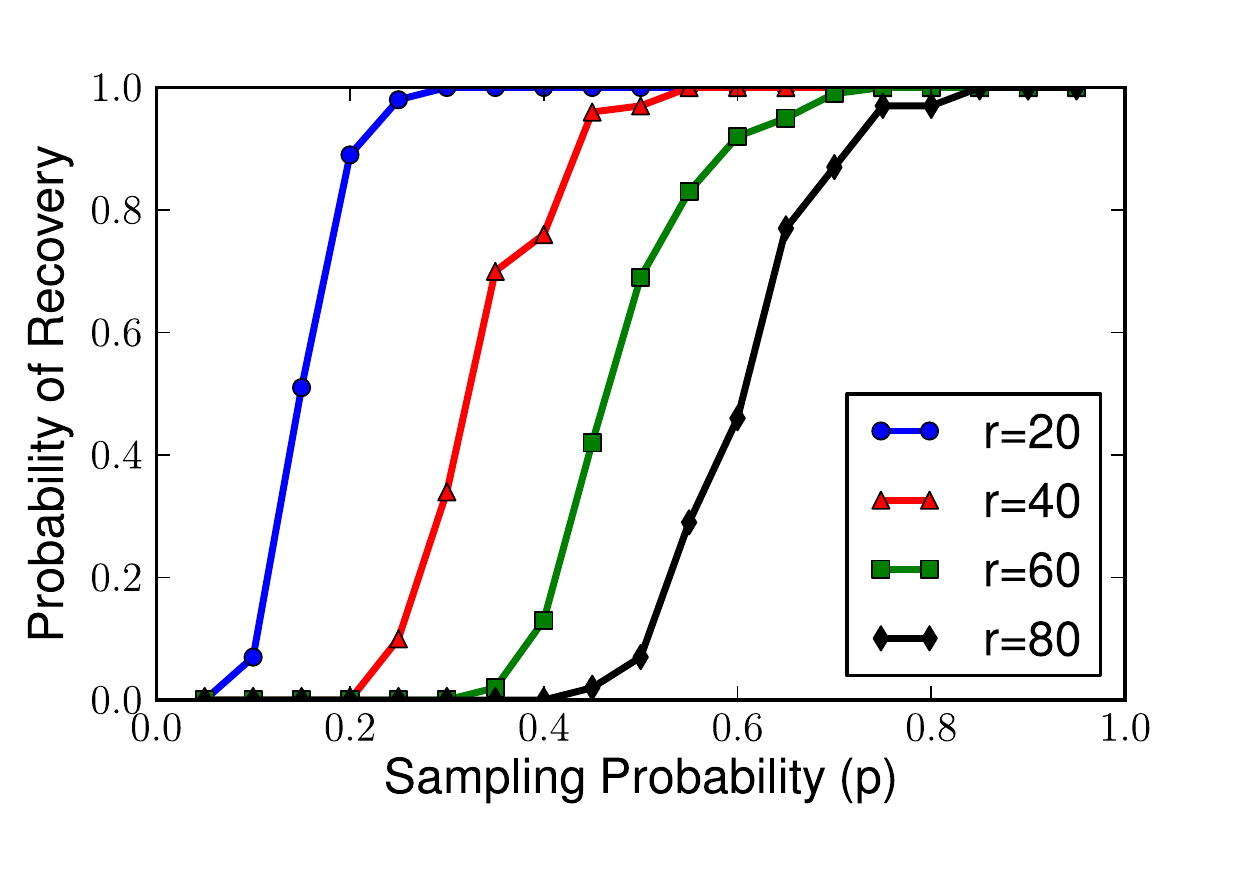}
\label{fig:unscaled_mc_r}
}
\subfigure[]{
\includegraphics[scale=\figsize]{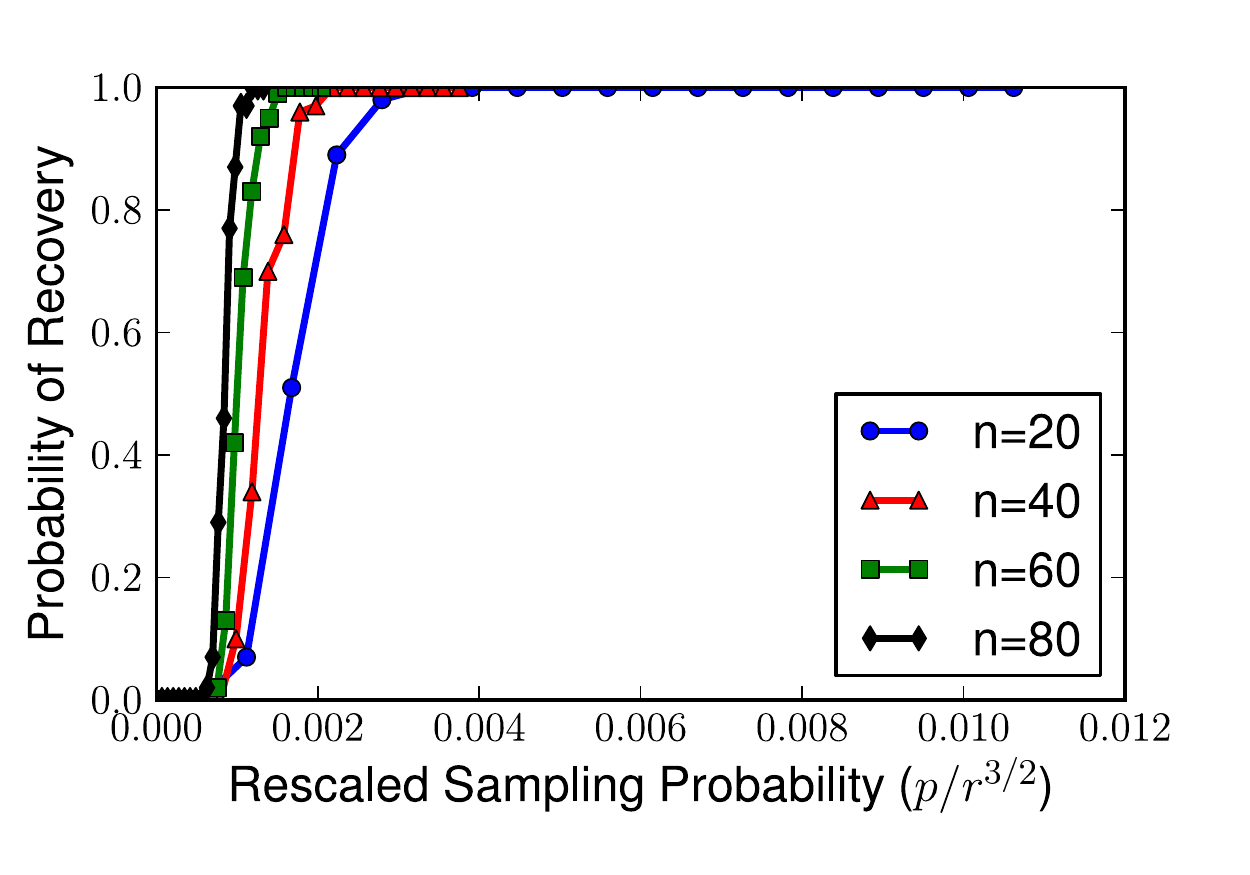}
\label{fig:three_halves_mc_r}
}
\subfigure[]{
\includegraphics[scale=\figsize]{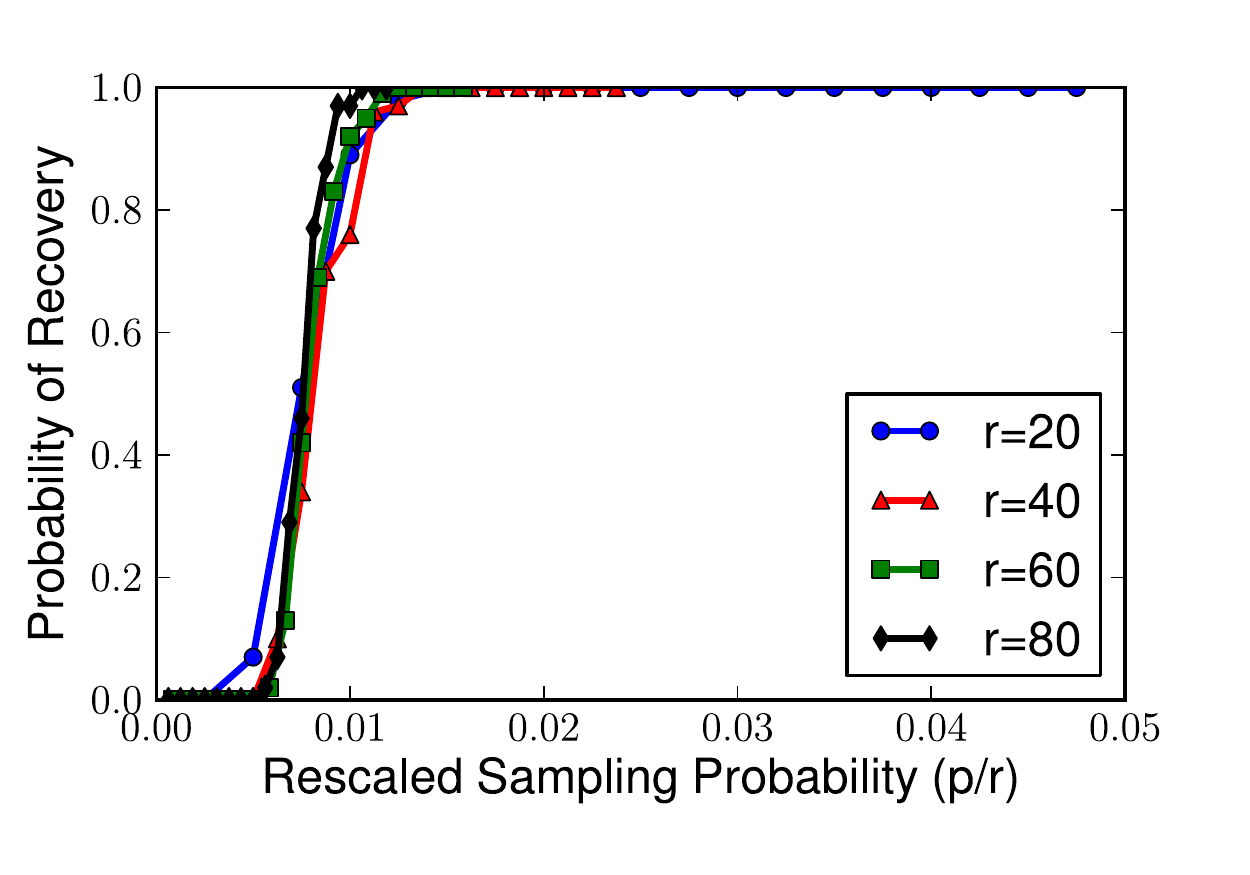}
\label{fig:linear_mc_r}
}
\vspace{-0.4cm}
}
\caption{Probability of success curves for our noiseless matrix completion algorithm (top) and SVT (middle).
Top: Success probability as a function of: Left: $p$, the fraction of samples per column, Center: $np$, total samples per column, and Right: $n p \log^2 n$, expected samples per column for passive completion.
Bottom: Success probability of our noiseless algorithm for different values of $r$ as a function of $p$, the fraction of samples per column (left), $p/r^{3/2}$ (middle) and $p/r$ (right).}
\label{fig:mc_thresholds}
\end{figure*}



We verify Corollary~\ref{cor:mc}'s linear dependence on $n$ in Figure~\ref{fig:mc_thresholds}, where we empirically compute the success probability of the algorithm for varying values of $n$ and $p=m/n$, the fraction of entries observed per column.
Here we study square matrices of fixed rank $r=5$ with $\mu(U)=1$.
Figure~\ref{fig:unscaled_mc} shows that our algorithm can succeed with sampling a smaller and smaller fraction of entries as $n$ increases, as we expect from Corollary~\ref{cor:mc}.
In Figure~\ref{fig:linear_mc}, we instead plot success probability against total number of observations per column.
The fact that the curves coincide suggests that the samples per column, $m$, is constant with respect to $n$, which is precisely what Corollary~\ref{cor:mc} implies.
Finally, in Figure~\ref{fig:logarithmic_mc}, we rescale instead by $n/\log^2 n$, which corresponds to the passive sample complexity bound \cite{recht2011simpler}.
Empirically, the fact that these curves do not line up demonstrates that our algorithm requires fewer than $\log^2 n$ samples per column, outperforming the passive bound.

The second row of Figure~\ref{fig:mc_thresholds} plots the same probability of success curves for the Singular Value Thresholding (SVT) algorithm \cite{cai2010singular}.
As is apparent from the plots, SVT does not enjoy a linear dependence on $n$;
indeed Figure~\ref{fig:logarithmic_svt} confirms the logarithmic dependency that
we expect for passive matrix completion, and establishes that our algorithm has
empirically better performance.

In the third row, we study the algorithm's dependence on $r$ on $500\times 500$ square matrices. 
In Figure~\ref{fig:unscaled_mc_r} we plot the probability of success of the algorithm as a function of the sampling probability $p$ for matrices of various rank, and observe that the sample complexity increases with $r$.
In Figure~\ref{fig:three_halves_mc_r} we rescale the $x$-axis by $r^{-3/2}$ so that if our theorem is tight, the curves should coincide.
In Figure~\ref{fig:linear_mc_r} we instead rescale the $x$-axis by $r^{-1}$ corresponding to our conjecture about the performance of the algorithm.
Indeed, the curves line up in Figure~\ref{fig:linear_mc_r}, demonstrating that empirically, the number of samples needed per column is linear in $r$ rather than the $r^{3/2}$ dependence in our theorem.

\begin{figure}[t]
\begin{floatrow}
\ffigbox{
\subfigure{
\includegraphics[scale=\figsizeb]{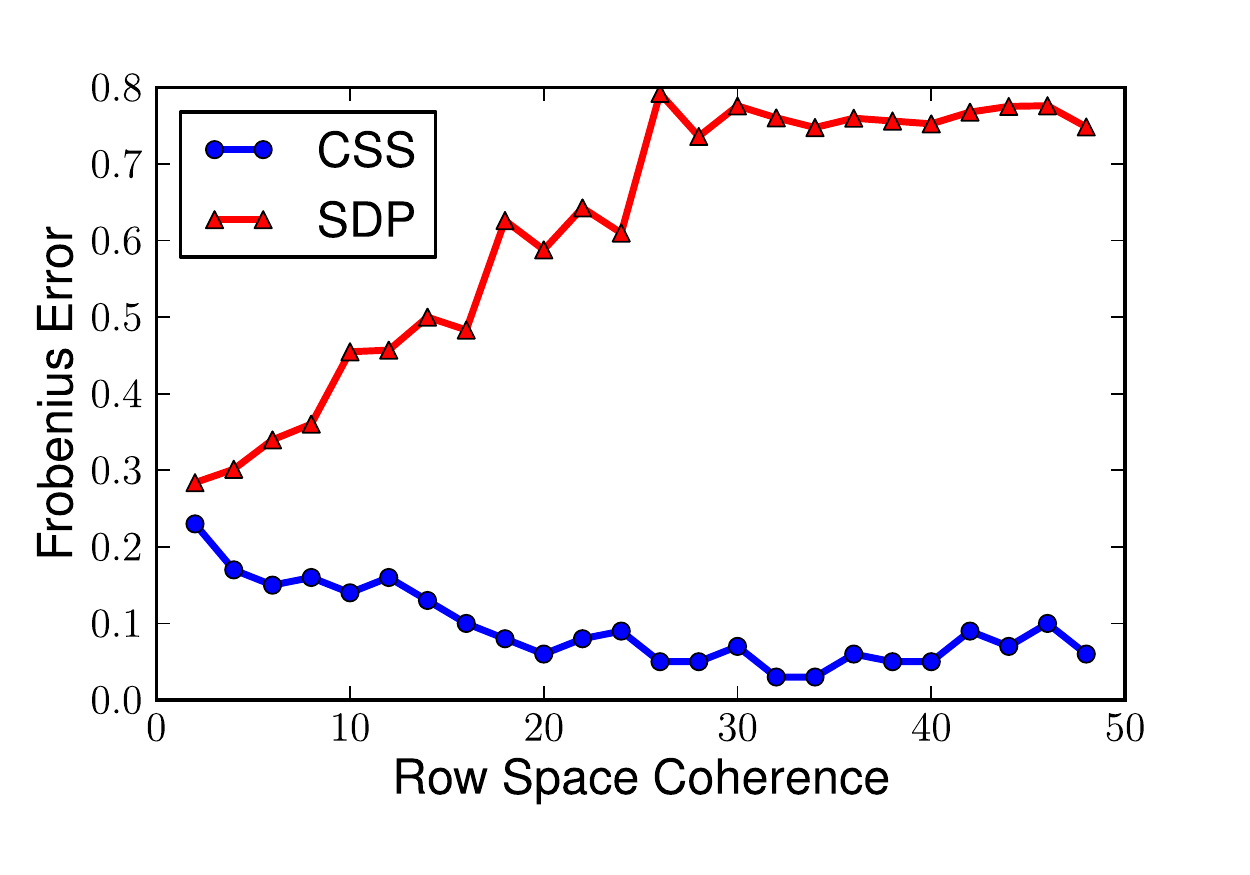}
\label{fig:linear_tc_3}
}
}{
\caption{Reconstruction error as a function of row space incoherence for our noisy algorithm (CSS) and the semidefinite program of~\cite{negahban2012restricted}.}
\label{fig:tc_thresholds}
}
\capbtabbox{
\begin{tabular}{c c c c | c }
\hline
\multicolumn{4}{c|}{Unknown $M$} & \multicolumn{1}{c}{Results}\\
\hline
$n$ & $r$ & $m/d_r$ & $m/n^2$ & time (s) \\ 
\multirow{3}{*}{$1000$} & $10$ & $3.4$ & $0.07$ & $16$ \\ 
& $50$ & $3.3$ & $0.33$ & $29$ \\ 
& $100$ & $3.2$ & $0.61$ & $45$ \\ 
\hline
\multirow{3}{*}{$5000$} & $10$ & $3.4$ & $0.01$ & $3$ \\ 
& $50$ & $3.5$ & $0.07$ & $27$ \\ 
& $100$ & $3.4$ & $0.14$ & $104$ \\ 
\hline
\multirow{3}{*}{$10000$} & $10$ & $3.4$ & $0.01$ & $10$ \\ 
& $50$ & $3.5$ & $0.03$ & $84$ \\ 
& $100$ & $3.5$ & $0.07$ & $283$ \\ 
\end{tabular}
}{\caption{Computational results on large low-rank matrices.
$d_r = r(2n-r)$ is the degrees of freedom, so $m/d_r$ is the oversampling ratio.}
\label{tab:computational_performance}
}
\end{floatrow}
\vspace{-0.5cm}
\end{figure}

To confirm the computational improvement over existing methods, we ran our matrix completion algorithm on large-scale matrices, recording the running time and error in Table~\ref{tab:computational_performance}.
To contrast with SVT, we refer the reader to Table 5.1 in~\cite{cai2010singular}.
As an example, recovering a $10000 \times 10000$ matrix of rank $100$ takes close to 2 hours with the SVT, while it takes less than 5 minutes with our algorithm.

For the noisy algorithm, we study the dependence on row-space incoherence. 
In Figure~\ref{fig:tc_thresholds}, we plot the reconstruction error as a function of the row space coherence for our procedure and the semidefinite program of Negahban and Wainwright~\cite{negahban2012restricted}, where we ensure that both algorithms use the same number of observations.
It's readily apparent that the SDP decays in performance as the row space becomes more coherent while the performance of our procedure is unaffected. 

%% file: discussion.tex
\vspace{-0.25cm}
\section{Conclusions and Open Problems}
\vspace{-0.15cm}
In this work, we demonstrate how sequential active algorithms can offer significant improvements in time, and measurement overhead over passive algorithms for matrix and tensor completion.
We hope our work motivates further study of sequential active algorithms for machine learning.


Several interesting theoretical questions arise from our work:
\vspace{-0.25cm}
\begin{packed_enum}
\item{} Can we tighten the dependence on rank for these problems?
In particular, can we bring the dependence on $r$ down from $r^{3/2}$ to linear?
Simulations suggest this is possible.
\item{} Can one generalize the nuclear norm minimization program for matrix completion to the tensor completion setting while providing theoretical guarantees on sample complexity?
\end{packed_enum}
\vspace{-0.25cm}
We hope to pursue these directions in future work.

%% file: appendix.tex
\section{Proof of Corollary~\ref{cor:mc}}
Corollary~\ref{cor:mc} is considerably simpler to prove than Theorem~\ref{thm:tc}, so we prove the former in its entirety before proceeding to the latter. 
To simplify the presentation, a number of technical lemmas regarding incoherence and concentration of measure are deferred to sections~\ref{sec:inc_props} and~\ref{sec:noise_props}, respectively.

The proof begins by ensuring that for every column $c_j$, if $c_j \notin \tU$ then $||(I - \mathcal{P}_{\tU_\Omega}) c_{j \Omega}||^2 > 0$ with high probability. This property is established in the following lemma:
\begin{lemma}
Suppose that $\tU \subset U$ and a column $c_j \in U$ but $c_j \notin \tU$. 
If $m \ge 36 r^{3/2}\mu_0 \log(2r/\delta)$ then with probability $\ge 1-4\delta$, $||(I - \mathcal{P}_{\tU_\Omega})c_{j \Omega}||^2 > 0$. 
If $c_j \in \tU$ then with probability $1$, $||(I - \mathcal{P}_{\tU_\Omega})c_{j \Omega}||^2= 0$.
\label{lem:our_lemma}
\end{lemma}
\begin{proof}[Proof of Lemma~\ref{lem:our_lemma}]
Decompose $c_j = x+v$ where $x \in \tU$ and $v \in \tU^\perp$.
We can immediately apply Theorem~\ref{thm:laura_new} and are left to verify that the left hand side of Equation~\ref{eq:detection_lwr_bd_new} is strictly positive.
Since $c_j \notin \tU$ we know that $||v||_2^2 > 0$.
Then:
\begin{eqnarray*}
\alpha & = & \sqrt{\frac{2 \mu(v)}{m}\log(1/\delta)} + \frac{2\mu(v)}{3m}\log(1/\delta) \le \sqrt{\frac{2 r \mu_0}{m} \log(1/\delta)} + \frac{2r \mu_0}{3m}\log(1/\delta) < 1/2\\
\end{eqnarray*}
When $m \ge 32 r \mu_0 \log(1/\delta)$.
Here we used that $\mu(v) \le r \mu(U)$ since $v \in \textrm{span}(U)$.
For $\gamma$:
\begin{eqnarray*}
\gamma & = & \sqrt{\frac{8 d \mu(\tU)}{3m} \log\left(\frac{2d}{\delta}\right)}
\le \sqrt{\frac{8 r\mu_0}{3m} \log\left(\frac{2r}{\delta}\right)} \le \frac{1}{3}
\end{eqnarray*}
Whenever $m \ge 24 r \mu_0 \log(2r/\delta)$.
Finally, with the bounds on $\alpha$ and $\gamma$, the expression in Equation~\ref{eq:detection_lwr_bd_new} is strictly positive when $3r \mu_0\beta \le m$ since $d\mu(\tU) \le r \mu_0$.
Plugging in the definition of $\beta$ we require:
\begin{eqnarray*}
6 \log(r/\delta) + \frac{4}{3}\frac{r^2 \mu_0}{m} \log^2(r/\delta) \le \frac{m}{3r \mu_0}
\end{eqnarray*}
Which certainly holds when $m \ge 36 r^{3/2}\mu_0 \log(r/\delta)$, concluding the proof.
\end{proof}

It is easy to see that if $c_i \in \tU$ then $||(I - \mathcal{P}_{\tU_{\Omega}})c_{i \Omega}||^2 = 0$ deterministically and our algorithm does not further sample these columns. 
We must verify that these columns can be recovered exactly, and this amounts to checking that $\tU_{\Omega}^T\tU_{\Omega}$ is invertible.
Fortunately, this was established as a lemma in \cite{balzano2010high}, and in fact, the failure probability is subsumed by the probability in Theorem~\ref{thm:laura_new}.
Now we argue for correctness: there can be at most $r$ columns for which $||(I - \mathcal{P}_{\tU_{\Omega}})c_{\Omega i}||^2 > 0$ since $\textrm{rank}(M) \le r$. 
For each of these columns, from Lemma~\ref{lem:our_lemma}, we know that with probability $1-4\delta$ $||(I - \mathcal{P}_{\tU_{\Omega}})c_{\Omega i}||^2 > 0$.
By a union bound, with probability $\ge 1 - 4r\delta$ all of these tests succeed, so the subspace $\tU$ at the end of the algorithm is exactly the column space of $M$, namely $U$.
All of these columns are recovered exactly, since we completely sample them.

The probability that the matrices $\tU_{\Omega}^T\tU_{\Omega}$ are invertible is subsumed by the success probability of Theorem~\ref{thm:laura_new}, except for the last matrix. 
In other words, the success of the projection test depends on the invertibility of these matrices, so the fact that we recovered the column space $U$ implies that these matrices were invertible.
The last matrix is invertible except with probability $\delta$ by Lemma~\ref{lem:invertibility}. 

If these matrices are invertible, then since $c_i \in \tU$, we can write $c_i = \tU\alpha_i$ and we have:
\begin{eqnarray*}
\hat{c}_i = \tU(\tU_{\Omega}^T\tU_{\Omega})^{-1}\tU_{\Omega}^T \tU_\Omega\alpha_i
= \tU\alpha_i = c_i
\end{eqnarray*}
So these columns are all recovered exactly.
This step only adds a factor of $\delta$ to the failure probability, leading to the final term in the failure probability of the theorem.

For the running time, per column, the dominating computational costs involve the projection $\mathcal{P}_{\tU_\Omega}$ and the reconstruction procedure. 
The projection involves several matrix multiplications and the inversion of a $r \times r$ matrix, which need not be recomputed on every iteration. 
Ignoring the matrix inversion, this procedure takes at most $O(n_1 r)$ per column for a total running time of $O(n_1n_2r)$. 
At most $r$ times, we must recompute $(U_{\Omega}^TU_{\Omega})^{-1}$, which takes $O(r^2m)$, contributing a factor of $O(r^3m)$ to the total running time. 
Finally, we run the Gram-Schmidt process once over the course of the algorithm, which takes $O(n_1r^2)$ time.
This last factor is dominated by $n_1n_2r$.

\section{Proof of Theorem~\ref{thm:tc}}
We first focus on the recovery of the tensor in total, expressing this in terms of failure probabilities in the recursion.
Then we inductively bound the failure probability of the entire algorithm. 
Finally, we compute the total number of observations. 
For now, define $\tau_{T}$ to be the failure probability of recovering a $T$-order tensor. 

By Lemma~\ref{lem:incoherence_properties}, the subspace spanned by the mode-$T$ tensors has incoherence at most $r^{T-2}\mu_0^{T-1}$ and rank at most $r$ and each slice has incoherence at most $r^{T-1}\mu_0^{T-1}$.
By the same argument as Lemma~\ref{lem:our_lemma}, we see that with $m \ge 36 r^{T-1/2}\mu_0^{T-1}\log(2r/\delta)$ the projection test succeeds in identifying informative subtensors (those not in our current basis) with probability $\ge 1- 4\delta$. 
With a union bound over these $r$ subtensors, the failure probability becomes $\le 4r\delta + \delta$, not counting the probability that we fail in recovering these subtensors, which is $r\tau_{T-1}$.

For each order $T-1$ tensor that we have to recover, the subspace of interest has incoherence at most $r^{T-3}\mu^{T-2}$ and with probability $\ge 1- 4r\delta$ we correctly identify each informative subtensor as long as $m \ge 36 r^{T-3/2}\mu^{T-2}\log(2r/\delta)$. 
Again the failure probability is $\le 4r\delta + \delta + r\tau_{T-2}$.

To compute the total failure probability we proceed inductively. 
$\tau_1 = 0$ since we completely observe any one-mode tensor (vector). 
The recurrence relation is:
\begin{eqnarray}
\tau_{t} = 4r\delta + \delta + r\tau_{t-1}
\end{eqnarray}
which solves to:
\begin{eqnarray}
\tau_{T} = \delta + 4r^{T-1}\delta + \sum_{t=1}^{T-2}5r^t\delta \le 5\delta
Tr^{T}
\end{eqnarray}

We also compute the sample complexity inductively.
Let $m_T$ denote the number of samples needed to complete a $T$-order tensor.
Then $m_1 = n_1$ and:
\begin{eqnarray}
m_t = r m_{t-1} + 36 n_t r^{t-1/2} \mu_0^{t-1} \log(2r/\delta)
\end{eqnarray}
So that $m_T$ is upper bounded as:
\begin{eqnarray*}
m_T &=& r^{T-1}n_1 + \sum_{t=2}^{T} r^{T-t} 36 n_t r^{t-1/2}\mu_0^{t-1}\log(2r/\delta) \le 36(\sum_{t=1}^T n_t)r^{T-1/2}\mu_0^{T-1}\log(2r/\delta)
\end{eqnarray*}

The running time is computed in a similar way to the matrix case.
Assume that the running time to complete an order $t$ tensor is:
\begin{eqnarray*}
O(r(\prod_{i=1}^tn_i) + \sum_{i=2}^tm_ir^{3+t-i})
\end{eqnarray*}
Note that this is exactly the running time of our Algorithm in the matrix case.

Per order $T-1$ subtensor, the projection and reconstructions take $O(r \prod_{t=1}^{T-1} n_t)$, which in total contributes a factor of $O(r \prod_{t=1}^T n_t)$.
At most $r$ times, we must complete an order $T-1$ subtensor, and invert the matrix $U_{\Omega}^TU_{\Omega}$. 
These two together take in total:
\begin{eqnarray*}
O\left(r\left[r(\prod_{t=1}^{T-1}n_t) + \sum_{t=2}^{T-1}m_tr^{3+T-1-t}\right] + r^3m_{T}\right)
\end{eqnarray*}
Finally the cost of the Gram-schmidt process is $r^2 \prod_{t=1}^{T-1}n_t$ which is dominated by the other costs. 
In total the running time is:
\begin{eqnarray*}
O\left(r \left(\prod_{t=1}^T n_t\right) + r^2\prod_{t=1}^{T-1}n_t + \sum_{t=2}^Tm_tr^{3+T-t}\right) = O\left(r \left(\prod_{t=1}^T n_t\right)+\sum_{t=2}^Tm_tr^{3+T-t}\right)\\
\end{eqnarray*}
since $r \le n_T$. 
Now plugging in that $m_i = \tilde{O}(r^{2(i-1)})$, the terms in the second sum are each $\tilde{O}(r^{T+t+1})$ meaning that the sum is $\tilde{O}(r^{2T+1})$.
This gives the computational result.

%% file: app_css.tex
\section{Proof of Theorem~\ref{cor:noisy_mc}}
We will first prove a more general result and obtain Theorem~\ref{cor:noisy_mc} as a simple consequence. 

\begin{theorem}
Let $M = A +R$ where $A= U\Sigma V^T$ and $R_{ij} \sim \mathcal{N}(0, \sigma^2)$.
Let $M_r$ denote the best rank $r$ approximation to $M$.
Assume that $A$ is rank $r$ and $\mu(U) \le \mu_0$.
For every $\delta, \epsilon \in (0,1)$ sample a set of size $s = \frac{5Lr}{2\delta \epsilon}$ at each of the $L$ rounds of the algorithm and compute $\hat{M}$ as prescribed.
Then with probability $\ge 1-9\delta$:
\[
||M - \hat{M}||_F^2\le 5/4 \left(\frac{1}{(1-\epsilon)} ||M-M_r||_F^2 + \epsilon^L||M||_F^2\right)
\]
and the algorithm has expected sample complexity:
\[
\Omega\left(\frac{L^2r}{\delta \epsilon}\left( n_1 + \mu_0 n_2 \sqrt{r} \log^2 \left(\frac{n_1 n_2 Lr}{\delta \epsilon}\right)\right)\right)
\]
\label{thm:noisy}
\end{theorem}

The proof of this result involves some modifications to the analysis in~\cite{deshpande2006matrix}.
We will follow their proof, allowing for some error in the sampling probabilities, and arrive at a recovery guarantee.
Then we will show how these sampling probabilities can be well-approximated from limited observations. 

The first Lemma analyzes a single round of the algorithm, while the second gives an induction argument to chain the first across all of the rounds.
These are extensions of Theorems 2.1 and Theorems 1.2, respectively, from~\cite{deshpande2006matrix}.

\begin{lemma}
Let $M = U\Sigma V^T \in \mathbb{R}^{n_1 \times n_2}$ and let $\tU$ be a subspace of $\mathbb{R}^{n_1}$.
Let $E = M - \mathcal{P}_{\tU}M$ and let $S$ be a random sample of $s$ columns of $M$, sampled according to the distribution $\hat{p}_i$ with:
\[
\frac{1-\alpha_1}{1+\alpha_2} \frac{||E_i||^2}{||E||_F^2} \le \hat{p}_i \le \frac{1+\alpha_2}{1-\alpha_1}\frac{||E_i||^2}{||E||_F^2}
\]
Then with probability $\ge 1-\delta$ we have:
\[
||M - \mathcal{P}_{\tU \cup \textrm{span}(S), r} M||_F^2 \le \frac{r}{s\delta} \frac{1+\alpha_2}{1-\alpha_1} ||E||_F^2 + ||M - M_r||_F^2
\]
Where $\mathcal{P}_{H, r}$ denotes a projection on to the best $r$-dimensional subspace of $H$ and $M_r$ is the best rank $r$ approximation to $M$.
\label{lem:css_inter}
\end{lemma}
\begin{proof}
The proof closely mirrors that of Theorem 2.1 in~\cite{deshpande2006matrix}. 
The main difference is that we are using an estimate of the correct distribution, and this will result in some additional error. 

For completeness we provide the proof here.
We number the left (respectively right) singular vectors of $M$ as $u^{(j)}$ ($v^{(j)}$) and use subscripts to denote columns.
We will construct $r$ vectors $w^{(1)}, \ldots, w^{(r)} \in \tU \cup \textrm{span}(S)$ and use them to upper bound the projection.
In particular we have:
\[
||M - \mathcal{P}_{\tU \cup \textrm{span}(S),r}M||_F^2 \le ||M - \mathcal{P}_{W}M||_F^2
\]
so we can exclusively work with $W$. 

For each $i=1, \ldots, n_2$ and for each $l=1, \ldots s$ define:
\[
X_l^{(j)} = \frac{1}{\hat{p}_i}E_iv^{(j)}_i \textrm{ with probability } \hat{p}_i
\]
That is the $i$th column of the residual $E$, scaled by the $i$th entry of the $j$th right singular vector, and the sampling probability. 
Defining $X^{(j)} = \frac{1}{s}\sum_{l=1}^sX_l^{(j)}$, we see that:
\[
\mathbb{E}[X^{(j)}] = \mathbb{E}[X_l^{(j)}] = \sum_{i=1}^{n_2} \frac{\hat{p}_i}{\hat{p}_i} E_iv^{(j)}_i = Ev^{(j)}
\]
Defining $w^{(j)} = \mathcal{P}_{\tU}(M)v^{(j)} + X^{(j)}$ and using the definition of $E$, it is easy to verify that $\mathbb{E}[w^{(j)}] = \sigma_j u^{(j)}$. 
It is also easy to see that $w^{(j)} - \sigma_ju^{(j)}= X^{(j)} - Ev^{(j)}$.

We will now proceed to bound the second central moment of $w^{(j)}$. 
\[
\mathbb{E}[||w^{(j)} - \sigma_ju^{(j)}||^2] = \mathbb{E}[||X^{(j)}||^2] - ||Ev^{(j)}||^2
\]
The first term can be expanded as:
\[
\mathbb{E}[||X^{(j)}||^2] 
= \frac{1}{s^2}\sum_{l=1}^s\mathbb{E}[||X_l^{(j)}||^2] + \frac{s-1}{s} ||Ev^{(j)}||^2
\]
So that the second central moment is:
\[
\mathbb{E}[||w^{(j)} - \sigma_ju^{(j)}||^2] = \frac{1}{s^2}\sum_{l=1}^s\mathbb{E}[||X_l^{(j)}||^2] - \frac{1}{s}||Ev^{(j)}||^2
\]
Now we use the probabilities $\hat{p}_i$ to evaluate each term in the summation:
\[
\mathbb{E}[||X_l^{(j)}||^2] = \sum_{i=1}^{n_2}\hat{p}_i \frac{||E^{(i)}v_i^{(j)}||^2}{\hat{p}_i^2} \le \sum_{i=1}^{n_2} \frac{(1+\alpha_2) v_i^{(j)2}||E||_F^2}{1-\alpha_1} = \frac{1+\alpha_2}{1-\alpha_1} ||E||_F^2
\]
This gives us an upper bound on the second central moment:
\[
\mathbb{E}[||w^{(j)} - \sigma_ju^{(j)}||^2] \le \frac{1}{s}\frac{1+\alpha_2}{1-\alpha_1}||E||_F^2
\]
To complete the proof, let $y^{(j)} = w^{(j)}/\sigma_j$ and define the matrix $F = (\sum_{j=1}^ky^{(j)}u^{(j)T})M$.
Since $y^{(j)} \in W$, the column space of $F$ is contained in $W$ so $||M - \mathcal{P}_W(M)||_F^2 \le ||M - F||_F^2$. 
\[
||M-F||_F^2 = \sum_{i=1}^r ||(M-F)v^{(i)}||^2 = \sum_{i=k+1}^{r}\sigma_{i}^2 + \sum_{i=1}^k||\sigma_iu^{(i)} - w^{(i)}||^2
\]
We now use Markov's inequality on the second term.
Specifically, with probability $\ge 1-\delta$ we have:
\begin{eqnarray*}
||M-F||_F^2 &\le& ||M - M_k||_F^2 + \frac{1}{\delta}\mathbb{E}[\sum_{i=1}^k||\sigma_iu^{(i)} - w^{(i)}||^2]
\le ||M - M_r||_F^2 + \frac{r}{\delta s}\frac{1+\alpha_2}{1-\alpha_1}||E||_F^2
\end{eqnarray*}
\end{proof}

\begin{lemma}
Suppose that $(1+\alpha_2)/(1-\alpha_1) \le c$ for some constant $c$ and for each of $L$ rounds of sampling.
Let $S_1, \ldots, S_L$ denote the sets of columns selected at each round and set $s = \frac{Lcr}{\delta \epsilon}$.
Then with probability $\ge 1-\delta$ we have:
\[
||M - \mathcal{P}_{\bigcup_{i=1}^L S_i, r}M||_F^2 \le \frac{1}{1-\epsilon}||M - M_r||_F^2 + \epsilon^L ||M||_F^2
\]
\label{lem:css_induct}
\end{lemma}
\begin{proof}
The proof is by induction on the number of rounds $L$.
We will have each round of the algorithm fail with probability $\delta/L$ so that the total failure probability will be at most $\delta$. 
The base case follows from Lemma~\ref{lem:css_inter}.
At the $l$th round, the same lemma tells us:
\[
||M - \mathcal{P}_{\bigcup_{i=1}^l S_i, r}M||_F^2 \le ||M - M_r||_F^2 + \frac{lcr}{s\delta} ||E||_F^2
\]
Plugging in our choice of $s$ and the definition of $E$:
\[
||M - \mathcal{P}_{\bigcup_{i=1}^l S_i, r}M||_F^2 \le ||M - M_r||_F^2 + \epsilon ||M - \mathcal{P}_{\bigcup_{i=1}^{l-1}S_i,r}M||_F^2
\]
and applying the induction hypothesis we have:
\[
||M - \mathcal{P}_{\bigcup_{i=1}^L S_i, r}M||_F^2 \le ||M - M_r||_F^2 + \epsilon (\frac{1}{1-\epsilon}||M - M_r||_F^2 + \epsilon^{L-1}||M||_F^2)
\]
which gives us the desired result.
\end{proof}

To complete the proof, we just need to compute how many observations are necessary to ensure that $(1+\alpha_2)/(1-\alpha_1) \le c$.
We can do this by manipulating Theorem~\ref{thm:laura_new} and upper bounding the incoherences of the subspaces throughout the execution of the algorithm.

\begin{lemma}
We have:
\[
\frac{2}{5} \frac{||E_i||_2^2}{||E||_F^2} \le \hat{p}_i \le \frac{5}{2} \frac{||E_i||_2^2}{||E||_F^2}
\]
with probability $\ge 1-6\delta$ as long as the expected number of samples observed per column $m$ satisfies:
\[
m = \Omega\left(\frac{L^2 r^{3/2}\mu(U)}{\delta \epsilon} \log^2(n_1n_2Lr/\delta \epsilon)\right)
\]
\label{lem:noisy_laura_use}
\end{lemma}
\begin{proof}
To establish the result, we will use the concentration results from Section~\ref{sec:noise_props} and the incoherence results form Section~\ref{sec:inc_props}. 
The goal will be to apply Theorem~\ref{thm:laura_new} with a union bound across all rounds and all columns, but we first need to bound the incoherences.

With a union bound, Lemma~\ref{lem:inc_column} shows that each column (once projected onto the orthogonal complement of one of the subspaces) has incoherence $O(r \mu(U) \log(n_1n_2L/\delta))$ with probability $\ge 1-\delta$. 
At the same time, Lemma~\ref{lem:inc_subspace} reveals that with probability $\ge 1-\delta$ all of the subspaces in the algorithm have incoherence at most $O(\mu(U) \log(n_1L/\delta))$.

We can now apply Theorem~\ref{thm:laura_new}. We will, as usual, take a union bound across all columns and all rounds, so each $\delta$ term in that lemma will be replaced with a $\delta/(n_1L)$.
Denote by $\tU_l$ the subspace projected onto during the $l$th round of the algorithm.
With $m$ as in the lemma, the condition that $m \ge 8/3 \textrm{dim}(\tU_l) \mu(\tU_l) \log\left(\frac{2rn_1L}{\delta}\right)$ is clearly satisfied, since $\textrm{dim}(\tU_l) \le \frac{L^2r}{\delta\epsilon}$ and $\mu(\tU_l) \le c \mu(U)\log(n_1L/\delta)$. 
We also have that:
\begin{eqnarray*}
\alpha &=& \sqrt{\frac{2 \mu(v)}{m} \log (\frac{n_1 L}{\delta})} + \frac{2}{3}\frac{\mu(v)}{m} \log(n_1 L/\delta)\\
&\le& c_1 \sqrt{\frac{r \mu(U) \log^2(n_1n_2L/\delta)}{m}} + c_2\frac{r\mu(U) \log^2(n_1n_2L/\delta)}{m} \le O(1)
\end{eqnarray*}
By boosting the size of $m$ by a constant, we can make $\alpha \le 1/4$. 
For $\gamma$ we have:
\[
\gamma = \sqrt{\frac{8\textrm{dim}(\tU_l)\mu(\tU_l)}{3m}\log(2\textrm{dim}(\tU_l)/\delta)} \le c \sqrt{\frac{L^2k}{\delta\epsilon} \frac{\mu(U)}{m} \log^2(\frac{n_1rL^3}{\delta^2\epsilon})} \le 1/3
\]
if we choose the constants correctly.
Finally we have:
\begin{eqnarray*}
\beta &=& 6 \log(n_1L\dim(\tU_l)/\delta) + \frac{4}{3}\frac{\dim(\tU_l) \mu(v)}{m} \log^2(n_1L\dim(\tU_l)/\delta)\\
& \le & \log(\frac{n_1r L^3}{\delta^2\epsilon}) + \frac{L^2 r^2 \mu(U)}{m \delta \epsilon} \log^3(\frac{n_1n_2rL^3}{\delta^2 \epsilon})
\end{eqnarray*}
which gives:
\[
\frac{\textrm{dim}(\tU_l) \mu(\tU_l)}{m} \frac{\beta}{(1-\gamma)} \le \frac{L^2r \mu(U)}{m \delta \epsilon} \log^2(\frac{n_1rL^3}{\delta^2 \epsilon}) + \frac{L^4r^3\mu(U)^2}{m^2 \delta^2 \epsilon^2} \log^4(\frac{n_1n_2rL^3}{\delta^2\epsilon}) \le O(1)
\]
again using our definition of $m$.
In particular, if we make this bound $\le 1/4$ we then have that:
\[
\frac{m}{n_1} (1-1/2)||v - \mathcal{P}_Sv||_2^2 \ge ||v_{\Omega} - \mathcal{P}_{S_{\Omega}}v_{\Omega}||_2^2 \ge \frac{m}{n_1} (1+1/4) ||v - \mathcal{P}_Sv||_2^2
\]
in which case:
\[
\hat{p}_i = \frac{||v_{i \Omega} - \mathcal{P}_{S_{\Omega}}v_{i \Omega}||_2^2}{\sum_i||v_{i \Omega} - \mathcal{P}_{S_{\Omega}}v_{i \Omega}||_2^2} \le \frac{5}{2}p_i
\]
along with the other direction. 
\end{proof}

We are essentially done proving the theorem. The total number of samples used is:
\[
n_2 m = \Omega\left(\frac{n_2 L^2 r^{3/2}\mu(U)}{\delta \epsilon} \log^3\left(\frac{n_1n_2L r}{\delta \epsilon}\right)\right)
\]

We also completely observe $\Omega(L^2r/\delta\epsilon)$ columns. 
In total this gives us the sample complexity bound in Theorem~\ref{thm:noisy}.
The failure probability is $\le 7\delta$ ($6\delta$ from Lemma~\ref{lem:noisy_laura_use} and $\delta$ from Lemma~\ref{lem:css_induct}).

So far we have recovered a subspace that can be used to approximate $M$.
Unfortunately, we cannot actually compute $\mathcal{P}_{U_L}M$ given limited samples.
Instead, for each column $c$, we compute $\hat{c} = U_L (U_{L \Omega}^T U_{L \Omega})^{-1} U_{\Omega L} c_{\Omega}$ and use $\hat{c}$ as our estimate of the column.
This is similar to another projection operation, and the error will only be a constant factor worse than before. 
\begin{lemma}
Let $c_i$ denote a column of the matrix $M$ and let $\hat{U}$ denote the subspace at the end of the adaptive algorithm.
Write $\hat{c} = \hat{U}(\hat{U}_{\Omega}^T\hat{U}_{\Omega})^{-1}\hat{U}_{\Omega}c$
Then with probability $\ge 1-2\delta$:
\[
||c - \hat{c}||^2 \le \left(1+ \frac{r\mu(\hat{U})\beta}{m (1-\gamma)^2}\right) ||\mathcal{P}_{\hat{U}^\perp}c||^2
\]
With $\beta$ and $\gamma$ defined as in Theorem~\ref{thm:laura_new}.
\end{lemma}
\begin{proof}
Decompose $c = x+y$ where $x \in \hat{U}$ and $y \in \hat{U}^\perp$. 
It's easy to see that $x = \hat{U}(\hat{U}_{\Omega}^T\hat{U}_{\Omega})^{-1} \hat{U}_{\Omega} x_{\Omega}$ so we are left with:
\[
||y - \hat{U}(\hat{U}_{\Omega}^T\hat{U}_{\Omega})^{-1}\hat{U}_{\Omega}y||^2 = ||y||^2 + ||\hat{U}(\hat{U}_{\Omega}^T\hat{U}_{\Omega})^{-1}\hat{U}_{\Omega}y||^2
\]
Because $y \in U^\perp$ so the cross term is zero.
The second term here is equivalant to:
\[
||(\hat{U}_{\Omega}^T\hat{U}_{\Omega})^{-1} \hat{U}_{\Omega}y||^2 \le ||(\hat{U}_{\Omega}^T\hat{U}_{\Omega})^{-1}||_2^2 ||\hat{U}_{\Omega} y||_2^2
\]
By Lemma 3 in~\cite{balzano2010high} the first term is upper bounded by $\frac{n_1^2}{(1-\gamma)^2m^2}$ while Lemma~\ref{lem:laura_2_new} reveals that the second term is upper bounded by $\beta\frac{m}{n_1^2}r \mu(\hat{U})||y||^2$.
Combining these two yields the result.
\end{proof}

We already showed that with our choice of $m$, the expression in the above Lemma is smaller than $5/4$.
Moreover the probability of failure simply adds $2\delta$ to the total failure probability.
Thus:
\[
||M - \hat{M}||_F^2 = \sum_{i} ||c_i - \hat{c}_i||_2^2 \le 5/4 ||M - \mathcal{P}_{U_T}M||_F^2
\]
and the last expression we bounded previously.

\subsection{Proving the Theorem}
To prove the main theorem, it is best to view $M$ as equal to $A$ on all of the unobserved entries.
In other words, if $\Omega$ is the set of all observations over the course of the algorithm, the random matrix $R$ is zero on $\Omega^C$.
Since we never observed $M$ on $\Omega^C$, we have no way of knowing whether $M$ was equal to $A$ on those coordinates.
It is therefore fair to write $M = A + R_{\Omega}$ where $R$ is zero on $\Omega^C$.

We expand the norm and then apply the main theorem:
\begin{eqnarray*}
||A - \hat{M}||_F^2 \le 3 ||M - \hat{M}||_F^2 + 3||R_{\Omega}||_F^2 \le \frac{5}{4}\left(\frac{3}{1-\epsilon}||M - M_r||_F^2 + 3\epsilon^L ||M||_F^2\right) + 3||R_{\Omega}||_F^2
\end{eqnarray*}
Now since $M_r$ is the best rank $r$ approximation to $M$ (in Frobenius norm) and since $A$ is rank $r$, we know that $||M - M_r||_F \le ||M - A||_F$.
With this substitution and setting $\epsilon=1/2, L=\log_2(n_1n_2)$ we will arrive at the result (below constants are denoted by $c$ and they change from line to line):
\begin{eqnarray*}
||A - \hat{M}||_F^2 &\le& c_1 ||M - A||_F^2 + \frac{c_2}{n_1n_2}||M||_F^2 + c_3 ||R_{\Omega}||_F^2
 \le \frac{c_1}{n_1n_2} ||A||_F^2 + c_2 ||R_{\Omega}||_F^2
\end{eqnarray*}
which holds as long as $n_1n_2$ is sufficiently large.

%% file: app_lower.tex
\section{Proof of Theorem~\ref{thm:lower_bound}}
We start by giving a proof in the matrix case, which is a slight variation of the proof by Candes and Tao \cite{candes2010power}.
Then we turn to the tensor case, where only small adjustments are needed to establish the result. 
We work in the Bernoulli model, noting that Candes' and Tao's arguments demonstrate how to adapt these results to the uniform-at-random sampling model.

\subsection{Matrix Case}
In the matrix case, suppose that $l_1 = \frac{n_1}{r}$ and $l_2 = \frac{n_2}{\mu_0 r}$ are both integers.
Define the following blocks $R_1, \ldots R_r \subset [n_1]$ and $C_1, \ldots C_r \subset [n_2]$ as:
\begin{eqnarray*}
R_i &=& \{l_1(i-1)+1, l_1(i-1)+2, \ldots l_1i\}\\
C_i &=& \{l_2(i-1)+1, l_2(i-1)+2, \ldots l_2i\}
\end{eqnarray*}
Now consider the $n_1\times n_2$ family of matrices defined by:
\begin{eqnarray}
\mathcal{M} = \{\sum_{k=1}^r u_kv_k^T | u_k = [1, \sqrt{\mu_0}]^n \otimes
\mathbf{1}_{R_k}, v_k = \mathbf{1}_{C_k}\}
\end{eqnarray}
$\mathcal{M}$ is a family of block-diagonal matrices where the blocks have size $l_1 \times l_2$.
Each block has constant rows whose entries may take arbitrary values in $[1, \sqrt{\mu_0}]$.
For any $M \in \mathcal{M}$, the incoherence of the column space can be computed as:
\begin{eqnarray*}
\mu(U) &=& \frac{n_1}{r} \max_{j \in [n_1]} ||\mathcal{P}_U e_j||_2^2 =
\frac{n_1}{r} \max_{k \in [r]}\max_{j \in [n_1]} \frac{(u_k^Te_j)^2}{(u_k^Tu_k)^2} 
\le \frac{n_1}{r} \max_{k \in [r]} \frac{\mu_0}{(n_1/r)} = \mu_0
\end{eqnarray*}
A similar calculation reveals that the row space is also incoherent with parameter $\mu_0$.

Unique identification of $M$ is not possible unless we observe at least one entry from each row of each diagonal block.
If we did not, then we could vary that corresponding coordinate in the appropriate $u_k$ and find infinitely many matrices $M' \in \mathcal{M}$ that agree with our observations, have rank and incoherence at most $r$ and $\mu_0$ respectively.
Thus, the probability of successful recovery is no larger than the probability of observing one entry of each row of each diagonal block.

The probability that any row of any block is unsampled is $\pi_1 = (1-p)^{l_2}$ and the probability that all rows are sampled is $(1-\pi_1)^{n_1}$.
This must upper bound the success probability $1-\delta$.
Thus:
\begin{eqnarray*}
-n_1\pi_1 \ge n_1 \log(1-\pi_1) \ge \log(1-\delta) \ge -2\delta
\end{eqnarray*}
or $\pi_1 \le 2\delta/n_1$ as long as $\delta < 1/2$. 
Substituting $\pi_1 = (1-p)^{l_2}$ we obtain:
\begin{eqnarray*}
\log (1-p) \le \frac{1}{l_2} \log\left(\frac{2\delta}{n_1}\right) = \frac{\mu_0
  r}{n_2}\log\left(\frac{2\delta}{n_1}\right)
\end{eqnarray*}
as a necessary condition for unique identification of $M$.

Exponentiating both sides, writing $p = \frac{m}{n_1n_2}$ and the fact that $1-e^{-x} > x - x^2/2$ gives us:
\begin{eqnarray*}
m \ge n_1\mu_0 r \log\left(\frac{n_1}{2\delta}\right) (1-\epsilon/2)
\end{eqnarray*}
when $\mu_0r/n_2 \log(\frac{n_1}{2\delta}) \le \epsilon < 1$.

\subsection{Tensor Case}
Fix $T$, the order of the tensor and suppose that $l_1 = \frac{n_1}{r}$ is an integer.
Moreover, suppose that $l_t = \frac{n_t}{\mu_0 r}$ is an integer for $1 < t \le T$.
Define a set of blocks, one for each mode and the family
\begin{eqnarray*}
B_i^{(t)} &=& \{l_t(i-1)+1, l_t(i-1)+2, \ldots, l_ti\} \forall i \in [r], t \in
[p]\\
\mathcal{M} & = & \left\{\sum_{i=1}^r \otimes_{t=1}^T a_i^{(t)}
  \left| \begin{aligned} &a_i^{(1)}
=[1,\sqrt{\mu_0}]^n \otimes \mathbf{1}_{B_i^{(t)}}\\ & a_i^{(t)}=
\mathbf{1}_{B_i^{(t)}}, 1 < t \le T\end{aligned} \right. \right\}
\end{eqnarray*}

This is a family of block-diagonal tensors and just as before, straightforward calculations reveal that each subspace is incoherent with parameter $\mu_0$.
Again, unique identification is not possible unless we observe at least one entry from each row of each diagonal block.
The difference is that in the tensor case, there are $\prod_{i \ne 1} l_i$ entries per row of each diagonal block so the probability that any single row is unsampled is $\pi_1 = (1-p)^{\prod_{i \ne 1} l_i}$.
Again there are $n_1$ rows and any algorithm that succeeds with probability $1-\delta$ must satisfy:
\begin{eqnarray*}
-n_1 \pi_1 \ge n_1 \log(1-\pi_1) \ge \log(1-\delta) \ge -2\delta
\end{eqnarray*}
Which implies $\pi_1 \le 2\delta/n_1$ (assuming $\delta < 1/2$).
Substituting in the definition of $\pi_1$ we have:
\begin{eqnarray*}
\log(1-p) \le \frac{1}{\prod_{i\ne j}l_i}\log \left(\frac{2\delta}{n_1}\right) =
\frac{\mu_0^{T-1}r^{T-1}}{\prod_{i\ne j}n_i}\log\left(\frac{2\delta}{n_1}\right)
\end{eqnarray*}

The same approximations as before yield the bound (as long as $\frac{\mu_0^{T-1}r^{T-1}}{\prod_{i\ne j} n_i}\log(\frac{n_1}{2\delta}) \le \epsilon < 1$):
\begin{eqnarray*}
m \ge n_1 \mu_0^{T-1}r^{T-1} \log\left(\frac{n_1}{2\delta}\right)(1-\epsilon/2)
\end{eqnarray*}

%% file: app_inc.tex
\section{Properties about Incoherence}
\label{sec:inc_props}
A significant portion of our proofs revolve around controlling incoherences of various subspaces used throughout the execution of the algorithms
The following technical lemmas will enable us to work with these quantities.
\begin{lemma}
Let $U_1 \subset \mathbb{R}^{n_1}, U_2 \subset \mathbb{R}^{n_2}, \ldots U_T
\subset \mathbb{R}^{n_T}$ be subspaces of dimension at most $d$, let $W_1
\subset U_1$ have dimension $d'$. Define $\mathbb{S} =
\textrm{span}(\{\otimes_{t=1}^T u_i^{(t)}\}_{i=1}^d)$. Then:
\begin{enumerate}[(a)]
\item $\mu(W_1) \le \frac{\textrm{dim}(U_1)}{d'}\mu(U_1)$.
\item{} $\mu(\mathbb{S}) \le d^{T-1} \prod_{i=1}^T\mu(U_i)$.
\end{enumerate}
\label{lem:incoherence_properties}
\end{lemma}
\begin{proof}
For the first property, since $W_1$ is a subspace of $U_1$, $\mathcal{P}_{W_1} e_j = \mathcal{P}_{W_1}
\mathcal{P}_{U_1} e_j$ so $||\mathcal{P}_{W_1} e_j||_2^2 \le ||\mathcal{P}_{U_1}
e_j||_2^2$. The result now follows from the definition of incoherence.

For the second property, we instead compute the incoherence of:
\begin{eqnarray*}
\mathbb{S}' = \textrm{span}\left(\left\{\otimes_{t=1}^T u^{(t)}\right\}_{u^{(t)} \in U_t
  \forall t}\right)
\end{eqnarray*}
which clearly contains $\mathbb{S}$. Note that if $\{u_i^{(t)}\}$ is an
orthonormal basis for $U_t$ (for each $t$), then the outer product of all
combinations of these vectors is a basis for $\mathbb{S}'$. We now compute:
\begin{eqnarray*}
\mu(\mathbb{S}')
=&&\frac{\prod_{i=1}^T n_i}{\prod_{t=1}^T\textrm{dim}(U_t)}
\max_{k_1 \in [n_1], \ldots, k_T \in [n_T]}
||\mathcal{P}_{\mathbb{S}'} (\otimes_{t=1}^T e_{k_t})||^2\\
=&& \frac{\prod_{i=1}^T n_i}{\prod_{t=1}^T\textrm{dim}(U_t)} \max_{k_1, \ldots, k_T} \sum_{i_1, \ldots,
  i_T} \langle \otimes_{t=1}^T u_{i_t}^{(t)}, \otimes_{t=1}^T e_{k_t}\rangle^2\\
= && \frac{\prod_{i=1}^T n_i}{\prod_{t=1}^T\textrm{dim}(U_t)} \max_{k_1, \ldots, k_T}\sum_{i_1,
  \ldots,i_T} \prod_{t=1}^T (u_{i_t}^{(t)T}e_{k_t})^2\\
=&& \frac{\prod_{i=1}^T n_i}{\prod_{t=1}^T\textrm{dim}(U_t)} \prod_{j=1}^T \max_{k_j}
\sum_{i=1}^{r}(u_i^{(t)T}e_{k_j})^2
\le \prod_{t=1}^T\mu(U_t)
\end{eqnarray*}

Now, property (a) establishes that $\mu(\mathbb{S}) \le \frac{r^T}{r}
\mu(\mathbb{S}')$ which is the desired result. 
\end{proof}

\begin{lemma}
Let $U$ be the column space of $M$ and let $V$ be some other subspace of dimension at most $n_1/2 - k$.
Let $v_i = \mathcal{P}_{V^\perp} c_i$ for each column $c_i$. 
Then with probability $\ge 1 - \delta$:
\[
\max_i \mu(v_i) \le 3k \mu(U) + 24 \log(2n_1n_2/\delta) = O(k \mu(U)\log(n_1n_2/\delta))
\]
\label{lem:inc_column}
\end{lemma}
\begin{proof}
Decompose $v_i = x_i + r_i$ where $x_i \in U \cap V^\perp$ and $r_i \in U^\perp \cap V^\perp$. 
Since each column is composed of a deterministic component living in $U$ and a random component, it must be the case that $r_i$ is a random gaussian vector living in $U^\perp \cap V^\perp$, which is a subspace of dimension at least $n_1 - d \ge n_1 - \textrm{dim}(U) - \textrm{dim}(V)$. 
We can now proceed with the bound:
\begin{eqnarray*}
\mu(v_i) &=& n_1 \frac{||v_i||_\infty^2}{||v_i||_2^2} \le 3n_1 \frac{||x_i||_\infty^2 + ||r_i||_{\infty}^2}{||x_i||_2^2 + ||r_i||_2^2}\\
& \le & 3n_1 \frac{||x_i||_\infty^2}{||x_i||_2^2} + 3n_1 \frac{||r_i||_\infty^2}{||r_i||_2^2}\\
& \le & 3k \mu(U) + \frac{6 \sigma^2 n_1 \log(2n_1n_2/\delta)}{\sigma^2 (n_1-d) - 2 \sigma^2 \sqrt{(n_1 - d) \log(n_2/\delta)}}
\end{eqnarray*}
For the second line, we used that $\frac{\sum_i a_i}{\sum_i b_i} \le \sum_i \frac{a_i}{b_i}$ whenever $a_i,b_i \ge 0$ which is the case here.
Finally we use Lemma~\ref{lem:chi_square} on the denominator, \ref{lem:max_gaus} on the numerator, and a union bound over all $n_2$ columns.
For $(n_1-d)$ sufficiently large (as long as $\sqrt{(n_1-d) \log n_2/\delta} \le (n_1-d)/4$) and if $d \le n_1/2$ we can bound as:
\[
3k \mu(U) + \frac{12 n_1 \log(2n_1n_2/\delta)}{n_1-d} \le 3k \mu(U) + 24 \log(2n_1n_2/\delta)
\]
\end{proof}

\begin{lemma}
Let $I_l = \bigcup_{i=1}^l S_i$ and let $U_l = \textrm{span}(\{c_i\}_{i \in I_l})$ as in the execution of the noisy algorithm.
If $|I_l|\le n_1/2$ then with probability $\ge 1-\delta$, for all $l \in [L]$, we have:
\[
\mu(U_l) = O(\mu(U) \log(n_1L/\delta))
\]
\label{lem:inc_subspace}
\end{lemma}
\begin{proof}
It is clear that $U_l \subset \textrm{span}(\{c_i\}_{i \in I_l}) \bigcup \textrm{span}(\{r_i\}_{i \in I_l})$ which will make things much easier to analyze.
Note that $\textrm{span}(\{c_i\}_{i \in I_l}) \subset U$ the original incoherent subspace and let $R_{I_l}$ denote the random matrix of columns corresponding to $I_l$.
We then have:
\begin{eqnarray*}
||\mathcal{P}_{U_l}e_i||_2^2 &\le& ||\mathcal{P}_Ue_i||_2^2 + ||\mathcal{P}_{U^\perp} \mathcal{P}_{R_{I_l}} e_i||_2^2 \le ||\mathcal{P}_Ue_i||_2^2 + ||\mathcal{P}_{R_{I_l}} e_i||_2^2\\
& \le & \frac{r \mu(U)}{n_1} + ||R_{I_l}||^2_2 ||(R_{I_l}^TR_{I_l})^{-1}||^2_2 ||R_{I_l}^Te_i||_2^2\\
& \le & \frac{r \mu(U)}{n_1} + \frac{(\sqrt{n_1} + \sqrt{|I_l|} + \epsilon)^2}{(\sqrt{n_1} - \sqrt{|I_l|} - \epsilon)^4}(|I_l| + 2\sqrt{|I_l| \log(1/\delta)} + 2 \log(1/\delta))
\end{eqnarray*}
Now if $|I_l| \le n_1/2$ and $\delta$ is not exponentially small, the contribution from the random matrix is:
\[
\frac{(\sqrt{n_1} + \sqrt{|I_l|} + \sqrt{2 \log(2/\delta)})^2}{(\sqrt{n_1} - \sqrt{|I_l|} - \sqrt{2 \log(2/\delta)})^4}(|I_l| + 2 \sqrt{|I_l| \log(1/\delta)} + 2 \log(1/\delta)) = O(|I_l| \log(1/\delta)/n_1)
\]
So the total incoherence will be (note that $\textrm{dim}(U_l) = |I_l|$ with probability $1$ since $|I_l| \le n_1/2$):
\[
\mu(U_l) = \frac{n_1}{|I_l|} \max_i ||\mathcal{P}_{U_l}e_i|_2^2 \le \frac{n}{|I_l|} \left(\frac{r\mu(U)}{n} + O(|I_l|/n)\right) = O(\frac{r}{|I_l|} \mu(U) + 1) = O(\mu(U)\log(1/\delta))
\]
The failure probability here is $\delta n_1 L$ if there are $L$ rounds so the incoherence is:
\[
\mu(U_l) = O(\mu(U) \log(n_1 L))
\]
\end{proof}

%% file: app_noise.tex
\section{A Collection of Concentration Results}
\label{sec:noise_props}
We enumerate several concentration of measure lemmas that we use throughout our proofs.
Many of these are well known results and we provide the references to their proofs.

\subsection{Proof of Theorem~\ref{thm:laura_new}}
We improve on the result of Balzano \emph{et al.}~\cite{balzano2010high} to establish Theorem~\ref{thm:laura_new}.
The proof parallels theirs but with improvements to two key Lemmas.
The improvement stems from using Bernstein's inequality in lieu of standard Chernoff bounds in the concentration arguments and carries over into our sample complexity guarantees.
Here we state and prove the two lemmas and then sketch the overal proof.

\begin{lemma}
With the same notations as Theorem~\ref{thm:laura_new}, with probability $\ge 1-2\delta$.
\begin{eqnarray}
(1-\alpha)\frac{m}{n}||v||_2^2 \le ||v_{\Omega}||_2^2 \le (1+\alpha)\frac{m}{n}||v||_2^2
\end{eqnarray}
\label{lem:laura_1_new}
\end{lemma}
\begin{proof}
The difference between Lemma~\ref{lem:laura_1_new} and Lemma 1 from~\cite{balzano2010high} is in the definition of $\alpha$.
Here we have reduced the relationship between $\mu(v)$ and $m$ from $\mu(v)^2/m$ to $\mu(v)/m$.
The proof is an application of Bernstein's inequality.

Let $X_i = v_{\Omega(i)}^2$ so that $\sum_{i=1}^m X_i = ||v_{\Omega}||_2^2$.
We can compute the variance and bound for $X_i$ as:
\[
\sigma^2 = \mathbb{E}[X_i^2] = \frac{1}{n}\sum_{i=1}^nv_i^4 \le \frac{1}{n}||v||_{\infty}^2||v||_2^2, \ \ 
M= \max |X_i| \le ||v||_{\infty}^2
\]
Now we apply Berstein's inequality:
\[
\mathbb{P}\left(\left|\sum_{i=1}^mX_i - \mathbb{E}[\sum_{i=1}^mX_i]\right| > t\right) \le 2 \exp\left(\frac{1}{2} \frac{-t^2}{m\sigma^2 + \frac{1}{3}Mt}\right)
\]
Noting that $\EE[\sum_{i=1}^mX_i] = \frac{m}{n}||v||_2^2$ and setting $t = \alpha \frac{m}{n}||v||_2^2$ the bound becomes:
\[
\mathbb{P}\left(\left|\sum_{i=1}^mX_i - \frac{m}{n}||v||_2^2\right| > \alpha\frac{m}{n}||v||_2^2\right) \le 2 \exp\left(\frac{-\alpha^2m||v||_2^2}{2n||v||_{\infty}^2(1+\alpha/3)}\right) \le 2 \exp\left(\frac{-\alpha^2 m}{2\mu(v)(1+\alpha/3)}\right)
\]
Finally plugging in the definition of $\alpha$ from the theorem shows that the right hand side is $\le 2\delta$.
\end{proof}

In similar spirit to Lemma~\ref{lem:laura_1_new} we can also improve Lemma 2 from~\cite{balzano2010high} using Bernstein's inequality:
\begin{lemma}
With the same notations as Theorem~\ref{thm:laura_new}, with probability at least $1-\delta$:
\begin{eqnarray}
||U_{\Omega}^Tv_{\Omega}||_2^2 \le \beta \frac{m}{n}\frac{d \mu(U)}{n}||v||_2^2
\end{eqnarray}
\label{lem:laura_2_new}
\end{lemma}
\begin{proof}
Again the improvement in our Lemma is in the expression $\beta$ where we have an improved dependence between $m$ and $\mu(y)$.
The proof is an application of Bernstein's inequality. Note that:
\begin{eqnarray*}
||U_{\Omega}^Tv_{\Omega}||_2^2 = \sum_{j=1}^d\left(\sum_{i\in\Omega} u_{ji}v_i\right)^2 = \sum_{j=1}^d\left(\sum_{i \in \Omega}X_{ji}\right)^2
\end{eqnarray*}
Where we have defined $X_{ji} = \sum_{k=1}^nu_{jk}v_j\mathbf{1}_{\Omega(i) = k}$.
We have:
\begin{eqnarray*}
\mathbf{E}[X_{ji}] = 0, \ \mathbf{E}[X_{ji}^2] = \frac{1}{n}\sum_{k=1}^n \left(u_{jk}v_k\right)^2 \triangleq \sigma_j^2, \ |X_{ji}| \le ||u_j||_{\infty}||v||_{\infty} \triangleq M
\end{eqnarray*}
We apply Bernstein's inequality and take a union bound, so that with probability $\ge 1-\delta$:
\begin{eqnarray*}
 \forall j=1,\ldots, d \ \ \sum_{i=1}^m X_{ji} &\le& \sqrt{2m\sigma_j^2\log(d/\delta)} + \frac{2}{3}M\log(d/\delta)\\
\sum_{j=1}^d \left(\sum_{i=1}^m X_{ji}\right)^2 &\le& 3 \left(2m\left(\sum_{j=1}^d\sigma_j^2\right) \log(d/\delta) + \frac{4}{9} dM^2 \log^2(d/\delta)\right)\\
\end{eqnarray*}
Notice that:
\[
\sum_{j=1}^d \sigma_j^2 = \frac{1}{n}\sum_{j=1}^d\sum_{i=1}^nu_{ji}v_i \le \frac{1}{n}\sum_{i=1}^nv_i^2\sum_{j=1}^du_{ji}^2 \le \frac{1}{n}||v||_2^2 d \mu(U)/n
\]
Notice also that $||u_j||_{\infty}^2 \le d \mu(U)/n$. 
Plugging in these bounds, with probability $\ge 1-\delta$:
\begin{eqnarray*}
||U_{\Omega}^Tv_{\Omega}||_2^2 &\le& 3 \left(2 \frac{m}{n} \frac{d\mu(U)}{n} ||v||_2^2 \log(d/\delta) + \frac{4}{9} \frac{d\mu(U)}{n} d ||v||_{\infty}^2 \log^2(d/\delta)\right)\\
& \le & \frac{m}{n} \frac{d\mu(U)}{n} ||v||_2^2 \left(6 \log(d/\delta) + \frac{4}{3}\frac{d\mu(v)}{m} \log^2(d/\delta)\right)
\end{eqnarray*}
Where we used that $||v||_{\infty}^2 \le ||v||_2^2\mu(v)/n$ via the definition of incoherence. 
\end{proof}

It will also be essential for these projections matrices to be invertible even with missing observations, as this will allow us to reconstruct columns of the matrix.
\begin{lemma}[~\cite{balzano2010high}]
Let $\delta > 0$ and $m \ge \frac{8}{3}r \mu_0 \log(2r/\delta)$, Then:
\begin{eqnarray}
||(U_{\Omega}^TU_{\Omega})^{-1}||_2 \le \frac{n}{(1-\gamma)m}
\end{eqnarray}
with probability $\ge 1-\delta$, provided that $\gamma < 1$. 
In particular $U_{\Omega}^TU_{\Omega}$ is invertible.
\label{lem:invertibility}
\end{lemma}

\begin{proof}[Proof of Theorem~\ref{thm:laura_new}]
Let $W_{\Omega}^TW_{\Omega} = (U_{\Omega}^TU_{\Omega})^{-1}$.
If Lemma~\ref{lem:invertibility} holds, $U_{\Omega}^TU_{\Omega}$ is invertible so 
\[
v_{\Omega}^TU_{\Omega}(U_{\Omega}^TU_{\Omega})^{-1}U_{\Omega}^Tv_{\Omega} = ||W_{\Omega}U_{\Omega}^Tv_{\Omega}||_2^2 \le ||W_{\Omega}||_2^2 ||U_{\Omega}^Tv_{\Omega}||_2^2 \le ||(U_{\Omega}^TU_{\Omega})^{-1}|| ||U_{\Omega}^Tv_{\Omega}||_2^2
\]

And therefore:
\[
||v_{\Omega} - \mathcal{P}_{U_{\Omega}}v_{\Omega}||_2^2 = ||v_{\Omega}||_2^2 - v_{\Omega}^TU_{\Omega}(U_{\Omega}^TU_{\Omega})^{-1}U_{\Omega}^Tv_{\Omega} \ge (1-\alpha)\frac{m}{n}||v||_2^2 - \frac{d \mu(U)}{n} \frac{\beta}{1-\gamma} ||v||_2^2
\]
yields the lower bound.
The upper bound follows from the same decomposition and Lemma~\ref{lem:laura_1_new}. 
\end{proof}


\subsection{Concentration for Gaussian Vectors and Matrices}
We will also need several concentration results pertaining to gaussian random vectors and gaussian random matrices. 
The first of these will help us bound the $\ell_2$ norm of a gaussian vector:
\begin{lemma}\cite{laurent2000adaptive}
Let $X \sim \chi_d^2$.
Then with probability $\ge 1-2\delta$:
\[
-2 \sqrt{d \log (1/\delta)} \le X - d \le 2\sqrt{d \log(1/\delta)} + \log(1/\delta)
\]
\label{lem:chi_square}
\end{lemma}

A gaussian random vector $r$ has incoherence that depends on $||r||^2_{\infty}$ so it is crucial that we can control the maximum of gaussian random variables.
\begin{lemma}
Let $X_1, \ldots, X_n \sim \mathcal{N}(0, \sigma^2)$.
Then with probability $\ge 1-\delta$:
\[
\max_i |X_i| \le \sigma \sqrt{2 \log(2n/\delta)}
\]
\label{lem:max_gaus}
\end{lemma}

Finally, we will be projecting on to perturbed subspaces so we will need to control the coherence of these subspaces.
The spectrum of the perturbation, will play a role in the coherence calculations.
\begin{lemma}\cite{vershynin2010introduction}
Let $R$ be a $n \times t$ whose entries are independent standard normal random variables.
Then for every $\epsilon \ge 0$, with probability $1-2\exp\{-\epsilon^2/2\}$, one has:
\[
\sqrt{n} - \sqrt{t} - \epsilon \le \sigma_{\min}(R) \le \sigma_{\max}(R) \le \sqrt{n} + \sqrt{t} + \epsilon
\]
\label{lem:spectral_norm}
\end{lemma}